\documentclass[12pt]{article}    
\usepackage[margin=.75in]{geometry}
\usepackage{setspace}

\usepackage{dsfont}
\usepackage{setspace}

\usepackage{algorithm}
\usepackage{algpseudocode}
\usepackage{graphicx}
\usepackage{tikz}
\usepackage{subcaption}
\usepackage{multirow}
\usepackage{enumitem}
\usepackage{amsmath}
\usepackage{amssymb}
\usepackage{makecell}
\usepackage{booktabs} 
\usepackage{mathtools}
\usepackage{amsthm}

\makeatletter
\let\prism@originalincludegraphics\includegraphics
\RenewDocumentCommand{\includegraphics}{O{}m}{%
  \IfFileExists{#2}{%
    \prism@originalincludegraphics[#1]{#2}%
  }{%
    \fbox{%
      \begin{minipage}[c][0.18\textheight][c]{0.9\linewidth}
        \centering
        Missing graphic\\
        \texttt{\detokenize{#2}}
      \end{minipage}%
    }%
  }%
}
\makeatother

\newcommand{\FIGURE}[3]{#1\caption{#2.\ #3}}
\newenvironment{APPENDICES}{\appendix}{}

\newtheorem{proposition}{Proposition}
\newtheorem{lemma}{Lemma}
\newtheorem{corollary}{Corollary}
\newtheorem{definition}{Definition}

\DeclareMathOperator*{\argmax}{arg\,max}

\usepackage[
  backend=biber,
  natbib=true,   
  style=apa,
  maxbibnames=99
]{biblatex}
\DeclareLabeldate{
  \field{date}
  \field{year}
  \field{eventyear}
  \field{origyear}
  \literal{nodate}
}

\begin{document}
\title{Logging Policy Design for Off-Policy Evaluation}
\author{%
  \makebox[\textwidth][c]{%
    \begin{minipage}[t]{0.3\textwidth}
      \centering
      Connor Douglas\\
      \small New York University\\
      \scriptsize \texttt{cpd8405@stern.nyu.edu}
    \end{minipage}%
    \hfill
    \begin{minipage}[t]{0.3\textwidth}
      \centering
      Joel Persson\\
      \small Spotify\\
      \scriptsize \texttt{joelpersson@spotify.com}
    \end{minipage}%
    \hfill
    \begin{minipage}[t]{0.3\textwidth}
      \centering
      Foster Provost\\
      \small New York University\\
      \scriptsize \texttt{fprovost@stern.nyu.edu}
    \end{minipage}%
  }
}

\date{May 2026}
\maketitle
\begingroup
\renewcommand{\thefootnote}{}
\footnotetext{\textbf{Acknowledgments.} This material is based upon work supported by the Fubon Center of Technology, Business and Innovation at NYU Stern and the National Science Foundation Graduate Research Fellowship under Grant No. DGE-2234660. We would like to thank Kyle Kretchsman for supporting this research at Spotify. We would also like to thank Todd Wasson for helpful discussions on practical problems and solution approaches. 
}
\addtocounter{footnote}{-1}
\endgroup

\begin{abstract}
\begin{spacing}{1.15}
Off-policy evaluation (OPE) estimates the value of a target treatment policy (e.g., a recommender system) using data collected by a different logging policy. It enables high-stakes experimentation without live deployment, yet in practice accuracy depends heavily on the logging policy used to collect data for computing the estimate. We study how to design logging policies that minimize OPE error for given target policies. We characterize a fundamental reward–coverage tradeoff: concentrating probability mass on high-reward actions reduces variance but risks missing signal on actions the target policy may take. We propose a unifying framework for logging policy design and derive optimal policies in canonical informational regimes where the target policy and reward distribution are (i) known, (ii) unknown, and (iii) partially known through priors or noisy estimates at logging time. Our results provide actionable guidance for firms choosing among multiple candidate recommendation systems. We demonstrate the importance of treatment selection when gathering data for OPE, and describe theoretically optimal approaches when this is a firm's primary objective. We also distill practical design principles for selecting logging policies when operational constraints prevent implementing the theoretical optimum.
\end{spacing}
\end{abstract}

\section{Introduction}

Off-policy evaluation (OPE) is a key tool when implementing data-driven decision-making systems (policies), such as recommender and content-selection systems. The aim of OPE is to evaluate the value of some ``target'' policy using data collected under some other ``logging'' policy \citep{uehara2022review}.\footnote{In the literature, target policies are also called evaluation policies and logging policies are called behavior policies or learning policies.} OPE is a crucial first step to improve decision-making systems in many settings, as it enables data-informed exploration of candidate treatment policies without deployment. This step in development is especially important in areas such as medicine \citep{murphy2001marginal}, digital advertising \citep{thomas2017predictive}, and recommendation systems \citep{chen2019industrial}, where online experimentation can be costly or unethical and also can degrade the experience or welfare of individuals.

As a simplified example illustration, consider a content recommendation platform. A key feature of the platform is the user's home page, which displays a prominent recommendation each time the user visits the page. The recommendation might be a playlist, podcast, audiobook, etc. The firm continually iterates on its recommendation systems to improve the user experience and account for new technical developments, dynamic content catalogs, evolving user preferences, and broader platform priorities \citep{fernandez-loria2022comparison}. Before live (``in vivo'') testing, candidate recommendation policies are evaluated offline via OPE on logged, historical data. The logging policy is the policy that served recommendations to users during the logging period, generating the historical data. Previous research and platforms companies have indicated that OPE accuracy can vary dramatically with the choice of logging policy, and that some logging policies require vastly more data to achieve a given level of estimation precision \citep{zhu_safe_2022, tucker2023mval, Wan2022Safe, persson2024}. Yet little is known about how to best design logging policies to support reliable and efficient OPE.

This paper aims to address that gap. We offer a theoretical explanation for the sensitivity of OPE to the logging policy and show how logging policies \textit{should} be designed to support statistically reliable and efficient OPE. Specifically, we optimize the action propensities during data collection to minimize the mean-squared error (MSE) of an IPW estimate of a target policy's value, which is the expected mean reward accrued under that policy. We focus on the IPW estimator because it is unbiased and completely determined by the action propensities of the logging policy and the target policy. Therefore, our design isolates the effect of the logging policy on our OPE estimate. 

Practical considerations complicate the design problem. The action space---the possible recommendations for a user---may be very large, and rewards may be very sparse and heterogeneous across users (e.g., users may click on certain recommendations infrequently even if they are good recommendations). This sparsity problem poses statistical challenges for OPE, yet is precisely what the value of personalization hinges on. Furthermore, the firm may or may not know the exact candidate policies it will want to evaluate with the logged data before it has been collected. This means the logging policy may need to cover a broad range of actions, since the target policy to be evaluated may not be known in advance. Finally, since the logging policy is actually deployed to serve recommendations, a firm typically wants the user experience during the logging policy to be good. In economic terms, the firm wants to avoid the cost of excessively random action selection, if possible. Informed by this, we study logging policy design under varying degrees of knowledge of the target policy and of reward probabilities, and analyze the reward accrual (value produced) under resulting logging policies.

The design of logging policies for OPE is the subject of a small but emerging body of research. Much of this work has focused on the design of ``safe'' logging policies, which aims to attain acceptable performance outcomes during the logging period while still ensuring reliable and efficient OPE. This problem is related to, but distinct from, the problem of bandit learning or other standard adaptive experimentation settings. In the bandit literature, a typical objective is to design adaptive policies that maximize cumulative reward by balancing exploration of action-reward distributions against exploitation of accumulated knowledge \citep{sutton1998reinforcement}. The aim of logging policy design for OPE is different. It is not to adaptively maximize reward, nor to identify the best policy in as few samples as possible. Instead, logging policy design for OPE aims to sample actions so that the value of one or more (potentially unknown) target policies can be estimated offline from the accrued data, with as low statistical error as possible. In addition, firms will generally want to maintain acceptable recommendation performance during the logging period. For a content streaming platform such as Netflix or Spotify, this means that the reward metric of interest (e.g., average streaming rate across users among impressed items) during the deployment of a logging policy should not be much lower than that of the status-quo production policy that delivers the personalized experience per user.

There is substantial literature on how to construct unbiased and efficient statistical estimators for OPE. This work typically treats the data as fixed and generated according to some existing and potentially unknown logging policy. Key estimators from this stream of work include the direct method, inverse propensity weighting (IPW) \citep{chesnaye2022introduction} and doubly-robust (DR) approaches \citep{dudik2014doubly}. In the sample limit, these estimators enjoy strong statistical guarantees, including asymptotic unbiasedness and bounded variance. These hold as long as the data and models involved in the estimation satisfy standard causal inference assumptions and regularity conditions \citep[see e.g.,][for an overview and technical treatment]{uehara2022review}. In finite sample sizes, however, the choice of logging policy has a dramatic effect on the bias and variance of OPE estimates. Since all practical settings have finite samples, this motivates studying how the logging policy's action probabilities should be chosen to minimize bias and variance of the OPE estimator.

\subsubsection*{Summary of contributions.} We present a unified framework for logging policy design with the objective of minimizing the mean squared error (MSE) of an IPW estimate of a target policy's value. Our framework organizes the design problem along two dimensions: knowledge of the target policy and knowledge of the reward distribution. This organization reflects the informational constraints faced in practice and yields a clean space of design problems with tractable optimal solutions.

Section 2 demonstrates the importance of logging policy choice for OPE. Section 3 then sets up the challenges for the design problem. We derive the bias-variance decomposition of the IPW estimator as a function of the logging policy, and show that a ``good'' logging policy must balance a fundamental \textit{reward-coverage} tradeoff: concentrating probability mass on high-reward actions reduces variance from reward realization, but coverage across actions the target policy may take is needed to control both bias and variance. Balancing this tradeoff is the central design challenge and connects logging policy design to classical problems in optimal experimental design \citep{wald1943efficient, kiefer1959optimum} and importance sampling. 

In Section 4, we derive optimal logging policies at the extrema of our space of informational settings. We show that when neither the target policy nor the reward distribution is known, uniform randomization is minimax optimal. This reflects the classical intuition that an uninformative prior is optimal when one assumes nothing about the estimand and wants to control worst-case error. We then consider the ideal case where both the target policy and reward function are known. We show the optimal logging policy takes the form of a Neyman allocation \citep{neyman1934representative} directed toward the target policy, weighting action probabilities by the product of target policy mass and the square root of the reward probability. This solution minimizes IPW variance subject to overlap, and yields two notable byproducts: the resulting OPE estimate based on IPW has lower MSE than the empirical on-policy estimate one would obtain by running the target policy directly, and the logging policy accrues higher expected reward than the target policy during the logging period. This confirms that OPE with an optimized logging policy can jointly attain higher policy value and estimation precision than on-policy evaluation (e.g., an A/B test).

In Section 5, we extend our analysis to intermediate informational settings that may better reflect practical applications. We present and prove the optimal solution for cases where only the distribution over target policies is known. Examlpes of such cases include when the firm plans to evaluate several candidate systems from a single logged dataset. For this case, the optimal logging policy should not allocate sampling mass towards any given target policy, but instead towards a second-moment pseudo-target. We provide a simple plug-in construction that instantiates this idea. We then consider the case that instead the reward distribution is only partially known. This occurs in practice when rewards are estimated, for instance by fitting a machine learning model on historically logged data. We show that using noisy reward predictions directly to in logging policy design inevitably inflates the variance of the downstream IPW estimate. As a solution, we propose the use of posterior shrinkage of reward estimates toward a per-context, across-action mean, and provide the optimal amount of shrinkage under a Gaussian hierarchical prior and noise. This prior can in turn be estimated from data, leading to an ``empirical Bayes'' \citep{robbins1956empirical, efron1973stein} interpretation of the logging policy design solution.

In Section 6, we provide practical guidance for settings where the theoretically optimal logging policy cannot be deployed. This analysis is motivated by organizational constraints. Fully personalized action propensities for actions may not be available at serving time, and deploying an optimized logging policy per target policy can carry excessive engineering costs. Irrespective of the reason, we highlight that even when the theoretically best logging policies cannot be implemented, simple ``\textit{soft-greedy}'' logging policies that weight action propensities toward current, in-production estimates can greatly improve the reliability and efficiency of OPE. We describe three suitable policy classes (top-k, softmax, and power-normalized), each indexed by a single \textit{greediness} parameter that smoothly interpolates between uniform randomization and deterministic greedy selection. By this construction, these policy classes balance the reward-coverage tradeoff through a single decision variable. We show through simulation that well-tuned policies in these classes can approach the MSE of the theoretical optimum, and provide guidance on how the greediness parameter should be chosen relative to sample size and action space.

\section{Setting}

We frame the setting similarly to a contextual bandit problem: a Markov Decision Process without transition dynamics, and where rewards from actions given contexts are realized immediately. This directly maps to our example of a content platform selecting one piece of content at a time to recommend to a particular user through its prominent home page slot. The reward of interest is some binary preference signal, which in our example could be whether a user chooses to stream the recommended content or not. This framing is standard in the literature and enables scalable deployment \citep{xie2018ope}. 

Let $\mathcal{X}$ denote the space of contexts and $\mathcal{A}$ the set of available actions. Here, context refers to a user and the circumstances under which they are being served an action (for example, time of day, day of week, etc.). In our example, actions would refer to the items of content to recommend. Contexts arrive according to a stationary distribution $p_X$, which may be unknown. Given a realized context $x\in\mathcal{X}$, the system selects an action $a\in\mathcal{A}$ according to a policy $\pi(\cdot\mid x)$. We then observe a reward $R$ drawn from a conditional distribution $p_{R\mid X,A}$. We take rewards to be binary variables $R\mid (A,X) \sim \text{Bernoulli}(\mu(A,X))$ where $\mu(a,x)\coloneqq \mathbb E[R\mid A=a,X=x]$ denotes the conditional mean reward, or, equivalently, the success probability $\Pr(R=1\mid A=a,X=x)$. In our running example, $\mu$ can be thought of as the chance that a given user streams a content item after being recommended it, given their context. 


Any policy, including the $\pi_l$, produces some average reward across contexts and the actions it selects. This expected reward is the policy's \emph{value} $V$. Formally, the value of a policy $V(\pi) = \mathbb{E}[R]$, with this expectation taken across context $X$ from $p_X$, actions $A$ drawn from $\pi( \cdot | X)$, and the Bernoulli reward $R$ drawn from $p_{R\mid A,X}$. The aim of off-policy evaluation is to construct an accurate and precise estimate of the policy value, $\hat{V}$, for some ``target'' policy, using data collected by the logging policy $\pi_l$. 

We treat the action probability (propensity score) $\pi(a\mid x)$ as known for each $a\in \mathcal A$ and $x \in \mathcal X$. This is not restrictive, since the task of logging policy design is precisely to determine these probabilities. Our setup satisfies the standard assumptions of consistency and unconfoundedness by construction. These assumptions are central in the causal inference literature \citep[e.g.,][]{hernan2010causal, imbens2015causal} and off-policy evaluation \citep[][Section~2]{uehara2022review}. Consistency means that the realized reward for a given context depends only on the action assigned in that context. This is reasonable in our motivating example of a homepage content recommender with a single recommendation slot and no cross-user interactions. It may be violated in settings such as multi-slot feeds or social platforms with network effects. Unconfoundedness requires that potential rewards are independent of action assignment conditional on the observed context. This holds by construction in our setting, since actions are randomized according to a logging policy chosen and implemented by the logging policy designer.

\subsection{Inverse Propensity Score Weighting (IPW)}

A key class of estimators of a target policy's value is derived via inverse propensity weighting (IPW). Formally, the IPW estimator of a target policy's value is defined as 
\begin{equation}
    \hat{V}_{\text{IPW}}(\pi_t | \pi_l) = \frac{1}{N}\sum_{i=1}^{N}\frac{\pi_t(A_i|X_i)}{\pi_l(A_i|X_i)}R_i,
\end{equation}
which is computed over the sample $\{X_i, A_i, R_i\}^N_{i=1} \overset{i.i.d.}\sim p_X(x)\pi_l(a | x)p_{R\mid X,A}(r\mid x,a)$ generated by deploying the logging policy $\pi_l$, and the sample size $N$ can be interpreted as the finite experimentation budget.

In this paper, we examine how to construct logging policies for off-policy evaluation using the IPW estimator. Our use of the IPW estimator is motivated by two features. \emph{First}, most popular OPE estimators build on the IPW estimator, including the self-normalized IPW and doubly robust variants \citep{dudik2014doubly}). \emph{Second}, the IPW estimator is completely \emph{design-based}: it requires only observed rewards and action propensities. This is a benefit in applications since the designer controls the sampling of actions but not why and how users respond to them. Thus, by considering the IPW estimator, we design logging policy precisely to maximize the accuracy of our OPE estimates. As we will see, however, designing an effective logging policy \emph{will} depend crucially on assumptions about rewards and the target policy. This is inevitable if the logged data is to be most informative for off-policy evaluation. An aim of our work is to make design considerations and their trade-offs explicit.

The problem of logging policy design is based on the observation that even with the IPW estimator held fixed, the choice of logging policy has a dramatic effect on the accuracy of off-policy evaluation. Figure~\ref{fig:estimate_histogram} shows the sampling distributions of IPW estimates under two logging policies, \textit{personalized} and \textit{uniform}, across repeated simulated datasets and two sample sizes. Here, a personalized logging policy assigns equal probability to the 10 actions with the highest noisy reward estimates; the uniform policy samples assigns equal probability across all 10,000 actions. Both policies are used to gather data for evaluation a target policy, which places mass uniformly across the 10 actions with the highest (noiseless) reward probabilities, with the IPW estimator. A logging policy that samples more frequently from high-reward items produces estimates whose sampling distribution is more tightly concentrated around the true target policy value. This effect is particularly pronounced in the small-sample regime, where sampling variability of the estimate is inevitably greater. Here, uniform logging can produce highly inaccurate off-policy estimates, often resulting in extremely small or extremely large estimates (see the spikes at zero and one for the small-sample regime for the uniform policy). Collecting more informative data substantially reduces this error. As the figure shows, $1{,}000$ observations under the personalized logging policy match the accuracy of $100{,}000$ observations under uniform logging.

This example illustrates the practical importance of logging policy design for off-policy evaluation. Because the logging phase represents a finite experimentation period prior to deployment, an effective logging policy can attain a desired level of evaluation accuracy with fewer observations, or equivalently, achieve greater accuracy for a fixed experimentation budget. In recommender systems, this is particularly important because prolonged experimentation can delay product innovations, expose users to suboptimal recommendations, and incur opportunity costs relative to deploying the target policy.


\begin{figure}[h!]
\FIGURE
    {
    \centering
    \includegraphics[width=10cm]{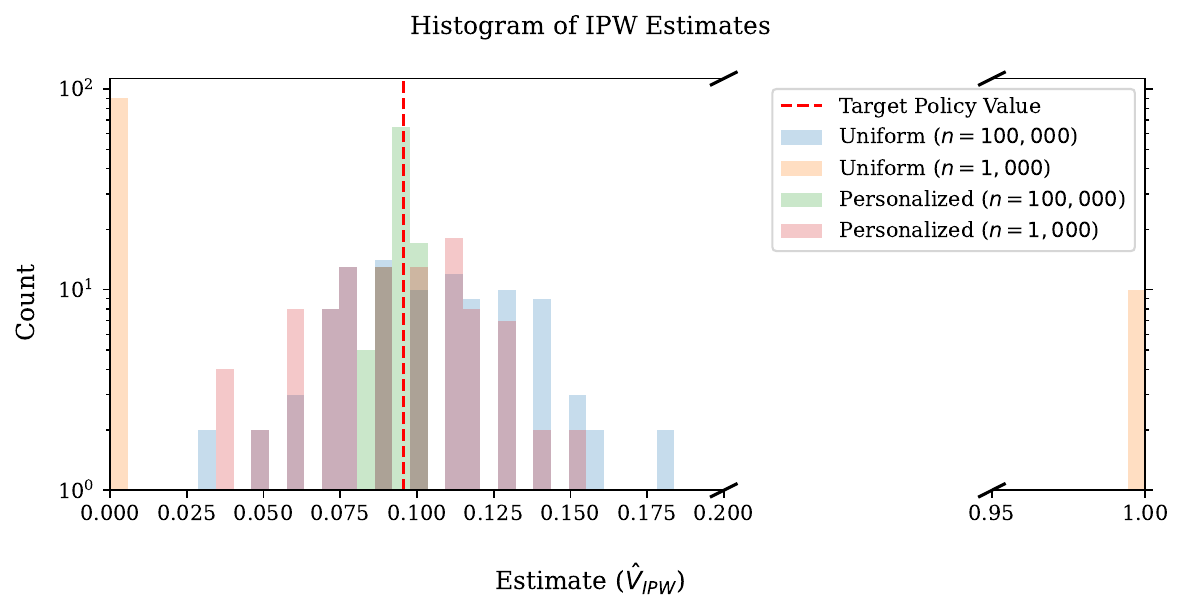}
    }
    {
    Dependence of IPW estimates on logging policy\label{fig:estimate_histogram}
    }
    {
    Histogram of $\hat{V}_{\text{IPW}}$ based on \textit{uniform} and \textit{personalized} logging policies, in small sample ($n=1,000$) and large sample ($n=100,000$) regimes. Uniform logging yields far worse estimates, especially in the small sample regime. Full simulation details are in Online Appendix \ref{app:simulations}
    }
\end{figure}


\subsection{Objective: Minimize MSE in Informational Settings}

The previous section illustrates that the sampling distribution of the IPW estimator depends critically on the logging policy. To compare logging policies, we therefore require a metric that evaluates the accuracy of the estimator $\hat V_{\mathrm{IPW}}(\pi_t \mid \pi_l)$ relative to the true target policy value $V(\pi_t)$, taking into account both systematic bias and sampling variability. A natural choice is the mean squared error (MSE), as it decomposes into bias and variance, represents the convex underlying cost structure of errors, and  is a standard measure of ``risk'' in OPE and statistical decision theory.

We therefore study logging policy design with the objective of minimizing the MSE of the IPW estimator, given by
\begin{equation}
    \mathrm{MSE}\!\left[\hat V_{\mathrm{IPW}}(\pi_t\mid \pi_l)\right]
    =
    \mathbb E 
    \Big[
    \big(  
    \hat{V}_{\text{IPW}}(\pi_t | \pi_l) - V(\pi_t)
    \big)^2
    \Big].
\end{equation}
where the expectation is taken over the joint distribution $p_X(x)\pi_l(a | x)p_{R\mid X,A}(r\mid x,a)$ of the logged data under $\pi_l$, that is, across arriving context, actions selected by the logging policy, and resultant rewards. 

Our focus on the MSE is also motivated by its differentiability and the generally convex nature of underlying costs to the designer with respect to errors in OPE estimates. Throughout the paper, we use the term \emph{error} as shorthand for MSE when convenient, and denote by $\pi^*_l$ a minimizer within some feasible class $\Pi$ of logging policies.


With this objective in mind, we study logging policy design under the different informational settings a designer may face when choosing the logging policy. Applications reside in a space characterized by two dimensions: \textit{knowledge of the target policy $\pi_t$} and \textit{knowledge of the reward distribution $\mu$}, as illustrated in Figure~\ref{fig:info_settings}. Because the data generated by the logging policy are not yet realized at the design stage, the designer must reason beforehand about estimation error, leading to objectives based on worst-case or average error, according to the information available a priori. In Section~4, we analyze logging policies in the extreme cases of full information and no information about these quantities, and in Section~5 we extend the analysis to intermediate settings.

\begin{figure}[h!]
\FIGURE
    {
    \centering
    \includegraphics[width=7cm]{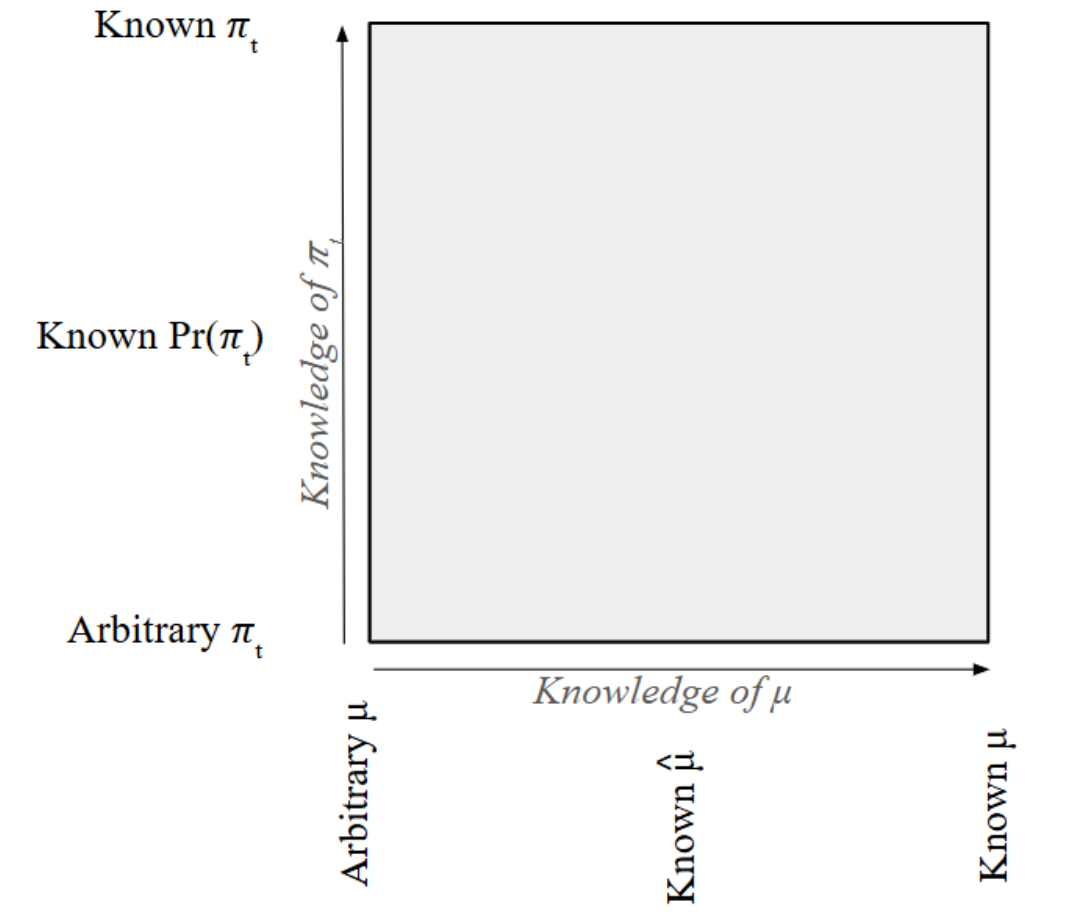}
    }
    {
    Informational settings for logging policy design
    \label{fig:info_settings}
    }
    {
    The two dimensions of information available to the logging policy designer: \textit{knowledge of the target policy $\pi_t$} (vertical axis), and \textit{knowledge of the reward distribution $\mu$ (horizontal axis)}
    }
\end{figure}


\section{MSE Decomposition and the Reward--Coverage Tradeoff}

To analyze the accuracy of OPE estimates under different logging policies, we need to understand the sources of error in the off-policy estimates. The MSE---or the risk, as it is known in statistical decision theory---of $\hat{V}_{\text{IPW}}$ can be decomposed into \textit{bias} and \textit{variance}. This decomposition is useful because the two sources call for different design strategies. Throughout, we treat context arrivals as exogenous and average over them, as they are outside the control of the logging policy designer.

\subsection{Bias-Variance Tradeoff}\label{sec:bias_variance}
Let $\mathcal{A}_x \coloneqq \{a \in \mathcal{A}, \pi_l(a|x)>0\}$ denote the support of $\pi_l$ for context $x$. Following the standard decomposition (see Online Appendix \ref{app:b_v_decomposition}), we have that
\begin{gather*}
\text{Bias}^2 \coloneqq \left( \sum_{x\in\mathcal X}\Pr(x)\sum_{a\notin\mathcal A_x}\pi_t(a\mid x)\mu(a,x)\right)^2,\\
\text{Var} \coloneqq \frac{1}{N}\left[
    \sum_{x \in \mathcal{X}} \Pr(x) \left(\left(\sum_{a \in \mathcal{A}_x} \frac{\pi_t(a \mid x)^2}{\pi_l(a \mid x)} \mu(a, x)\right)\right.\right.\\
     - \left.\left. \left(\sum_{a \in \mathcal{A}_x} \pi_t(a \mid x)\mu(a, x)\right)^2\right)\right],\\
     \mathrm{MSE}\!\left(\hat V_{\mathrm{IPW}}(\pi_t\mid \pi_l)\right)=\text{Bias}^2 + \text{Var}.
\end{gather*}
In words, a logging policy will induce \textit{bias} in $\hat{V}$ when it assigns zero probability to actions with non-zero expected reward that the target policy $\pi_t$ may select. \emph{Variance}, in turn, captures error in the target policy's IPW value estimate due to random sampling, both in terms of sampling actions and the sampling of rewards from actions.

\begin{definition}[Weak Overlap \citep{uehara2022review}]
    A logging policy $\pi_l$ is said to have \textbf{weak overlap} with respect to a target policy $\pi_t$ if $\pi_l(a|x) > 0 \quad \forall a \in \mathcal{A}, x \in \mathcal{X} \text{ where } \pi_t(a|x)>0$. 
 \end{definition}



Weak overlap ensures that $V(\pi_t)$ is point-identified from the data collected under $\pi_l$, and is required for the IPW estimator to be unbiased. The assumption is called ``weak'' to distinguish it from the standard (strong) overlap assumption, which instead invokes that $\pi_l(a|x) > 0$ for all $a \in \mathcal A$ \citep[][Section 2]{uehara2022review}. Strong overlap ensures unbiased evaluation of \emph{any} target policy from the same logged data. Since the goal of logging policy designs is to collect data for a specific target policy, weak overlap with respect to that target is sufficient to ensure an unbiased estimate of $V(\pi_t)$. Since we will be working in this setting, we henceforth refer to weak overlap simply as overlap.

While overlap eliminates bias, it does not necessarily minimize MSE. When overlap fails, the IPW estimator instead targets the value of the target policy restricted to the logging support $\mathcal A_x$, inducing bias relative to $V(\pi_t)$. As shown in Online Appendix \ref{app:examples}, in small-sample or large-action settings, the increase in bias induced by this truncation can be more than offset by a substantial reduction in variance, yielding lower MSE with respect to the true policy value.

More generally, when a context arrives infrequently and the action space is large, the designer may reduce MSE by excluding certain actions from the logging support. This is because the MSE can explode as $\pi_l(a|x)$ approaches $0$, whereas bias is bounded. This implies that if the designer is concerned with minimizing MSE of the IPW estimator for a target policy's value, they may prefer logging policies that induce bias.

There are practical concerns with leaving out actions and incurring this bias, however. Leaving actions out of the support of the logging policy eliminates any signal on certain context-action pairs the logging policy designer may wish to estimate. Since variance decreases with sample size while bias does not, the optimal set of actions to sample depends on sample size, among other factors that may be outside the designer's control..

\subsection{Choice of Action Set}

Having established that a biased estimator may, in certain cases, yield a lower MSE, we now provide a sufficiency condition on when to include an action in the logging support. This condition takes the form of a lower bound on the probability of the logging policy recommending an action to a given user. If an action is sampled with at least this probability, then the MSE will be lower than if the action were left unsampled (i.e. $\pi_l(a|x)=0$).

\begin{proposition}[Sufficiency for sampling]\label{prop:sampling_sufficiency}
    A sufficient condition guaranteeing that the inclusion of $a$ in $\mathcal{A}_x$ for some context $x$ decreases MSE relative to a logging policy that assigns zero probability to action $a$ at context $x$ is if $\pi_l(a|x)>\frac{1}{\mu(a, x)(n\Pr(x)+1)}$.
\end{proposition}

\begin{proof} 
See Online Appendix \ref{app:sampling_sufficiency}.
\end{proof}

This requirement serves as a guide guaranteeing that an item should be sampled. However, this is not a necessary condition. If the logging probability falls below this threshold, removing the action from the support does not necessarily improve MSE. These cases call for an exact analysis of the logging policy and the current bias and variance it induces.

For our derivation of optimal solutions, we place an overlap requirement on the logging policy to avoid the complex analyses induced by these edge cases. Beyond the complexity in calculating a solution to these cases, these solutions are themselves uninformative: when a logging policy is designed so that errors are concentrated in the bias term, it effectively foregoes actually gathering any signal about the quality of the target policy. For instance, if a logging policy only collects data on actions that $\pi_t$ would never sample, the resulting IPW estimator produces an estimate of zero. This incurs bias but achieves zero variance by ignoring all collected data. In the small-sample or large-action regimes that are most interesting for applications, such a logging policy can yield a lower MSE than one that produces an unbiased but high-variance estimate due to small logging propensities, simply because variance then dominates the MSE. However, this type of trivial ``optimization'' misses the aim of off-policy evaluation which is to \emph{use} data collected under one policy to estimate the value of another. Since these solutions do not provide signal on the quality of a target policy to OPE, we will focus our theoretical analyses on the problem of variance minimization, placing appropriate constraints on our claims. Specifically, in our theoretical analyses, we restrict the search space of logging policies to those satisfying overlap. We relax this restriction in Section~\ref{sec:restrictions}, where we examine logging policy classes that may intentionally truncate the action support.


\subsection{Reward-Coverage Tradeoff}\label{sec:reward_coverage_tradeoff}

In addition to the standard bias-variance tradeoff, there is another fundamental tradeoff that must be balanced when deciding how to allocate probability mass across actions. From the derivation in Section \ref{sec:bias_variance}, we see that increasing mass on context-action pairs in $\pi_t$ reduces  MSE, as $\text{Variance} \propto \frac{\pi_t^2(a|x)}{\pi_l(a|x)}$, while adding coverage (support) to previously unsupported pairs can ensure bias reduction as well. This incentivizes \textit{coverage} across pairs on which $\pi_t$ places mass. However, we also have that $\text{Variance} \propto \frac{\mu(a, x)}{\pi_l(a|x)}$, so there is an incentive to place more mass on high \textit{reward} pairs in $\mu$. When $\pi_t$ and $\mu$ are misaligned, this creates a tradeoff that must be managed when assigning mass in $\pi_l$. As $\pi_t$ and $\mu$ become more aligned, this tension diminishes.

To illustrate this tradeoff, consider the following example. Let there be a simple, one context $(\mathcal{X} = \{x_1\})$, two action $(\mathcal{A}=\{a_1,a_2\})$ setting, with the logging policy defined via a scalar value as $\pi_l \coloneqq \pi_l(a_1|x_1)$. In this setting, suppose reward probabilities across items for the given user follow $\mu(a_1,x_1)=.9,\; \mu(a_2,x_1)=.1$. So using our running content platform example, in this case there is only one user (or one type of user), and two different possible streaming recommendations. One recommendation is very likely to be streamed, while the other is unlikely.

Continuing, the designer aims to devise a logging policy to evaluate one of two potential target recommendation policies offline from the collected data. In the first scenario, the target policy is \textit{aligned} with rewards: $\pi_t(a_1|x_1) = .9,\pi_t(a_2|x_1)=.1 $; in other words, it is more likely to recommend the content that is more likely to be streamed, as illustrated in Figure~\ref{fig:aligned}(a). In this case, we see from the figure that the optimal logging policy $\pi_l^*$ places nearly all its mass on $a_1$, yielding $\pi_l^* \approx .96$. In the second scenario, as illustrated in Figure~\ref{fig:aligned}(b), the target policy is \textit{misaligned} with rewards, $\pi_t(a_1|x_1) = .1,\pi_t(a_2|x_1)=.9$. Therefore, the logging designer must balance placing mass on those high-reward actions with placing mass on the actions likely to be selected by the target policy. For this setting, we see that the optimal logging policy sets $\pi_l^* = .25$. Note that in neither case is the optimal logging policy the same as the target policy. Therefore, even if the firm could run the target policy in vivo via an A/B test, doing so would not be the optimal strategy for evaluating its performance.

This example demonstrates the need to balance two competing objectives: allocating mass on high-reward actions (\textit{reward}) and on actions the target policy may take (\textit{coverage}). Matching the target policy's distribution minimizes variance due to action sampling, making it a natural objective. However, error also arises from variance in reward realization. Rewards are stochastic, and positive outcomes may be rare. For observations with zero-valued rewards ($r_i = 0$), the absolute error for observation $i$ is simply $\pi_t(a_i,x_i)\mu(a_i,x_i)$. If positive rewards ($r_i = 1$), the error depends inversely on $\pi_l$, and equals $\frac{\pi_t(a_i,x_i)}{\pi_l(a_i,x_i)}-\pi_t(a_i,x_i)\mu(a_i,x_i)$. A large value of $\mu(a_i,x_i)$ makes $r_i=1$ a more likely event, so small it is much more likely to produce $r_i = 1$, yielding an error which explodes in small $\pi_l(a_i,x_i)$ causes this error contribution to explode. This \textit{reward-coverage tradeoff} is the central challenge in logging policy design.

\begin{figure}[ht]
\FIGURE
    {
    \centering
    \begin{subfigure}[t]{0.48\textwidth}
        \centering
        \includegraphics[height=2.5in]{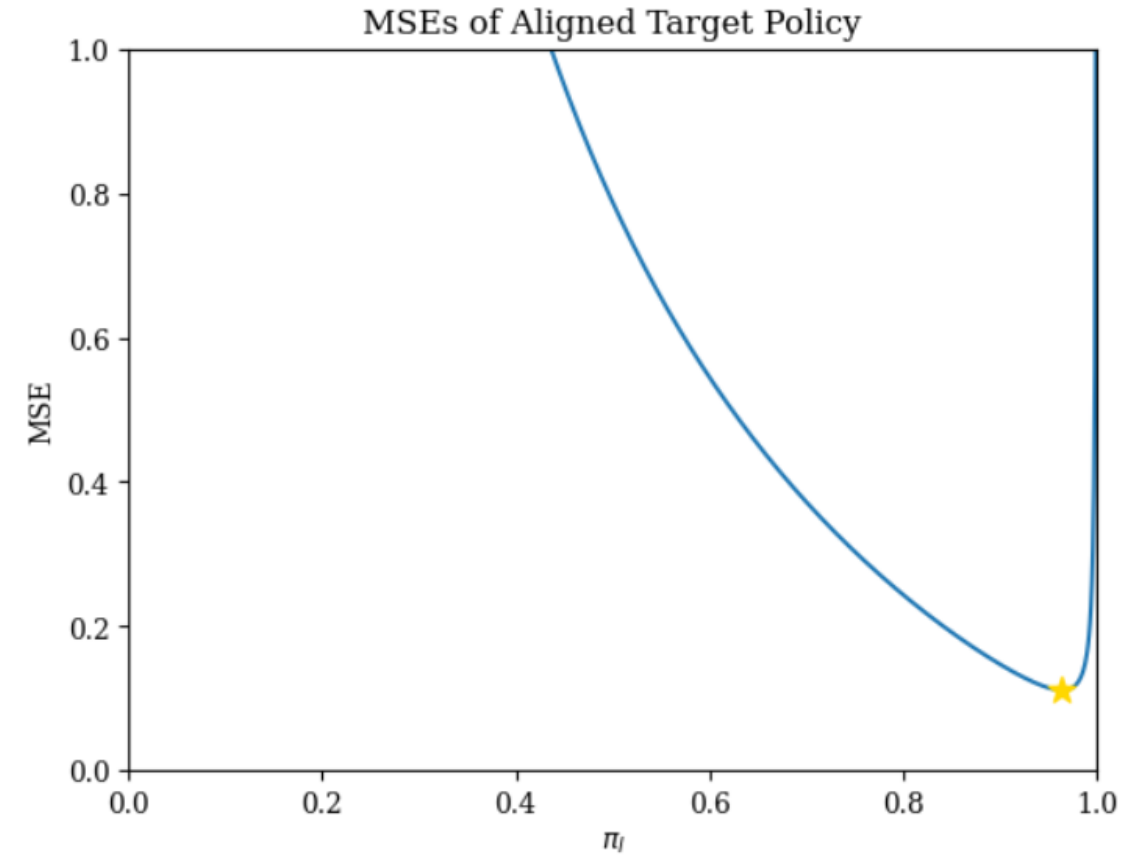}
        \caption{\emph{Aligned} setting}
    \end{subfigure}
    \hfill
    \begin{subfigure}[t]{0.48\textwidth}
        \centering
        \includegraphics[height=2.5in]{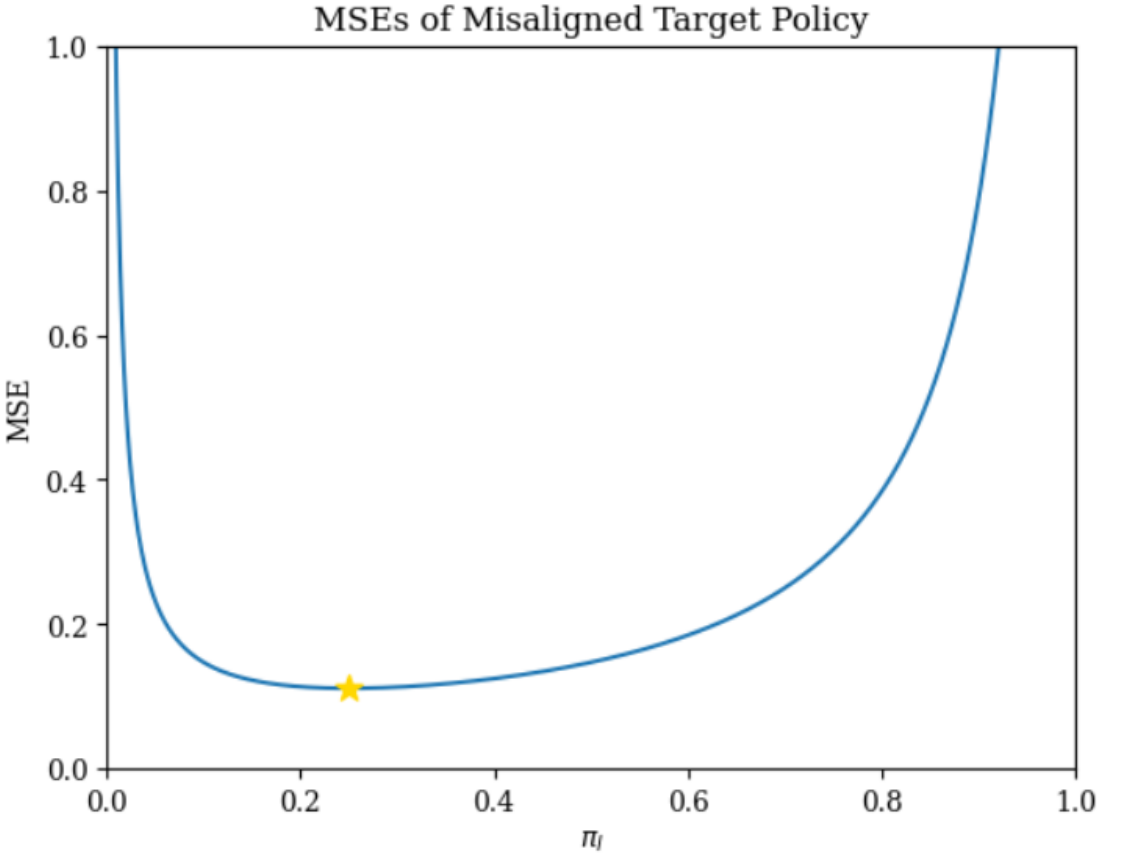}
        \caption{\emph{Misaligned} setting}
    \end{subfigure}
    }
    {
    Illustration of error across logging policy choices in \textit{aligned} and \textit{misaligned} target policies.\label{fig:aligned}
    }
    {
    Panel \emph{(a)} shows the MSE in the \textit{aligned} setting $\hat{V}_{IPW} (\pi_t|\pi_{l,\tilde{\mu}})$ across $\pi_l \coloneqq \pi_l(a_1|x_1)$ with $\mu(a_1,x_1)=.9$, $\mu(a_2,x_1)=.1$ and  $\pi_t(a_1|x_1) = .9$, $\pi_t(a_2|x_1)=.1 $.The MSE-minimizing $\pi_l^*$ is denoted as a star ($\pi_l^* \approx .96$). 
    Panel \emph{(b)} shows the MSE in the \textit{misaligned} setting $\hat{V}_{IPW} (\pi_t|\pi_{l,\tilde{\mu}})$ across $\pi_l \coloneqq \pi_l(a_1|x_1)$ with $\mu(a_1,x_1)=.9$, $\mu(a_2,x_1)=.1$ and $\pi_t(a_1|x_1) = .1$, $\pi_t(a_2|x_1)=.9$. The MSE-minimizing $\pi_l^*$ is denoted as a star ($\pi_l^* = .25$).
    }
\end{figure}


\section{Logging Policy Design at Informational Extrema}

Based on these general considerations for logging policy design, we now examine optimal logging policies in informational extrema (see Figure \ref{fig:info_settings}). We restrict attention to logging policies that induce an unbiased IPW estimator $\hat{V}_{\text{IPW}}(\pi_t \mid \pi_l)$, both to keep the analysis theoretically meaningful and to reflect the practical objective of collecting data to estimate $V(\pi_t)$.

For these no-knowledge settings, either $\mu$, $\pi_t$, or both are unknown. In light of this, we treat the design problem as an adversarial game: the logging policy designer aims to choose $\pi_l$ to minimize MSE with respect to $\pi_t$, while nature selects $\mu$ and $\pi_t$ to maximize MSE. This minimax formulation is standard in statistical decision theory and yields logging policies that are worst-case optimal. We first analyze the case where the target policy is unknown, then turn to the case where reward information is unavailable.


\subsection{Unknown $\pi_t$}

Under a completely unknown $\pi_t$, the only way to ensure overlap is to require full support across actions. With this requirement, we describe optimal logging policy design under \textit{unknown} and \textit{known} reward estimates.

\subsubsection{Unknown $\mu$}
Under settings where the designer has no information about $\pi_t$ and no information about $\mu$, the designer is only aware of the set of users on their platform and the set of items they can recommend. In this setting, the designer will minimize the MSE of their worst possible estimate by uniformly sampling across items. This is described below in Proposition \ref{prop:unknown_mu_unknown_pit}.

\begin{proposition}\label{prop:unknown_mu_unknown_pit}
    For an unconstrained target policy class $\Pi$, and a logging policy requiring full support, $\pi^*_l = \pi_U$ is the designer's solution to $\min\limits_{\pi_l \in \Pi} \max\limits_{\mu \in \mathcal{M},\pi_t \in \Pi} \mathrm{MSE}[\hat{V}_{\text{IPW}}(\pi_t | \pi_l)]$.
\end{proposition}

\begin{proof} 
See Online Appendix \ref{app:proof_uniform}.
\end{proof}

Here, we recover the intuition from statistical inference that uniform random sampling is the best way to collect data when we have no knowledge of the estimand of interest. In that sense, uniform random sampling of actions can be seen as optimal under an uninformative prior over the target policy. 

Unfortunately, this approach can yield very high MSE and cause undesirable actions to be selected during logging. In our running example, this would mean selecting content for users uniformly at random, regardless of their preferences or the target policy to be evaluated. Under even modest content diversity, this approach will recommend items users are unlikely to want and that a personalized target policy will take. As a result, the logging policy may be poorly aligned with both the reward distribution and the target policy, leading to high variance in OPE estimates. Furthermore, by recommending  content uniformly at random, the platform could incur a significantly degraded user experience.

Of course, that is not what platforms would (or should) do in practice. This error can be mitigated as the logging policy designer will generally have an implicit or explicit prior arising from information about rewards in the sampling process. As soon as this information is known, uniform randomization ceases to be the optimal logging policy.

\subsubsection{Known $\mu$}

As the designer learns more about their setting, they can design a solution that improves on uniform sampling. Consider the case where the designer knows the probabilities of reward for each user-item pair. In this case, the designer should recommend higher-reward items more often, designing a logging policy that is increasing in $\mu(\cdot,x)$. Online Appendix \ref{app:proofs} presents this worst-case analysis formally and shows how under known reward distribution $\mu$, a $\pi_l$ (ensuring overlap) that allocates mass proportional to $ \frac{\mu}{c+\mu^2}$ minimizes the worst-case MSE, proving the following proposition.

\begin{proposition}\label{prop:known_mu_unknown_pit}
    For an unconstrained target policy class $\Pi$, and a logging policy requiring full support, $ \pi_l^*(a|x) = \frac{\mu(a, x)}{c + \mu^2(a,x)}$ with $1 = \sum_{a \in \mathcal{A}_x}\frac{\mu(a, x)}{c + \mu^2(a,x)}$  is the designer's solution to the problem $\min\limits_{\pi_l \in \Pi} \max\limits_{\pi_t \in \Pi}\;\mathrm{MSE}[\hat{V}_{\text{IPW}}(\pi_t | \pi_l)]$.
\end{proposition}

\begin{proof} 
See Online Appendix \ref{app:proof_known_mu_unknown_pit}.
\end{proof}

In this solution, $c$ is a normalizing constant, ensuring $\pi_l^*(\cdot|x)$ comprises a valid probability distribution. Unfortunately, this constant cannot be solved for analytically, instead requiring numerical estimation. While it may not be immediately clear, this policy preserves the ordering of actions by $\mu(\cdot|x)$ in $\pi_l^*(\cdot|x)$. In our running example, another important implication is that such a policy would be ``safer" that the uniform policy, as it would be less likely to recommend content that the user would not like.

\begin{corollary}[Preservation of Ordering]\label{cor:preservation_of_ordering}
    For all $a, a' \in \mathcal{A} \quad s.t. \quad \mu(a, x) > \mu(a', x), \quad \pi^*_l(a|x) > \pi_l^*(a'|x).$
\end{corollary}

\begin{proof} 
See Online Appendix \ref{app:proof_preservation_of_ordering}.
\end{proof}

This claim builds the intuition that higher-reward actions should be sampled more. We now proceed from the logging policy designer having no knowledge of $\pi_t$ to having complete knowledge of the target policy.

\subsection{Known $\pi_t$}

In practice, it is common for the set of target policies (to be evaluated) to be constrained. At the extreme, there is one target policy to evaluate and the logging policy designer knows it.  In this case our results have the cleanest interpretation.


At first, designing a logging policy for this setting may seem unnecessary: if there is a single, known target recommendation policy, why not just run this policy \emph{in vivo}? FWhy not conduct a live A/B test of the target policy and evaluate it \textit{on-policy}? 

Answering this question is not straightforward: as Figure~\ref{fig:aligned} shows, the optimal logging policy may or may not be the target policy itself. To provide definite answers, we organize our analysis by whether the reward function is known.

\subsubsection{Unknown $\mu$}

In the setting where the designer has no knowledge of $\mu$, but full knowledge of $\pi_t$, the designer will indeed minimize worst-case MSE by collecting data under $\pi_t$ itself.

\begin{proposition}\label{prop:uknown_mu_known_pit}
    For a policy class $\Pi$ requiring full support on the support of $\pi_t$, $\pi^*_l = \pi_t$ is the designer's solution to $\min\limits_{\pi_l \in \Pi}\max\limits_{\mu \in \mathcal{M}} \;\mathrm{MSE}[\hat{V}_{\text{IPW}}(\pi_t | \pi_l)]$.
\end{proposition}

\begin{proof} 
See Online Appendix \ref{app:proof_uknown_mu_known_pit}.
\end{proof}

While the guarantee of optimality of this result is useful in these cases, it is not wholly surprising. However, in many applicable cases, this policy can be improved upon.

\subsubsection{Known $\mu$}

In practice, the reward function $\mu$ is rarely entirely unknown at the time of deciding on a logging policy. Consider the case where $\mu$ is fully observed: the designer knows the probabilities that users will like each item. Using $\pi_t$ to log data will generally yield a sub-optimal solution. The optimal solution is presented in Proposition \ref{prop:known_mu_known_pit}.

\begin{proposition}\label{prop:known_mu_known_pit}
    For known $\pi_t$ and $\mu$ and a logging policy requiring full support on the support of $\pi_t$, setting $\pi^*_l(a|x) =\frac{\pi_t(a|x)\sqrt{\mu(a, x)}}{\sum_{a'\in \mathcal{A}}\pi_t(a'|x)\sqrt{\mu(a', x)}}$ minimizes $\mathrm{MSE}[\hat{V}_{\text{IPW}}(\pi_t | \pi_l)]$.
\end{proposition}

\begin{proof} 
See Online Appendix \ref{app:proof_known_mu_known_pit}.
\end{proof}

This result aligns with the optimality of Neyman allocation\footnote{The Neyman allocation describes strata size selection in stratified sampling according to outcome variance and population size \citep{neyman1934representative}.} in stratified sampling and optimal proposal selection in importance sampling, but the connection to optimal design for OPE via IPW has not yet been shown in the literature. 

It is worth noting that, in the case of a deterministic $\pi_t$, we do recover that the optimal logging policy for MSE-minimization in the policy value estimate is $\pi_t$, meaning that on-policy evaluation is the optimal approach to data collection. In our running example, a deterministic policy would always recommend the same streaming content to a user in a particular context. Such ``degenerate" policies are rarely used in practice as they are counter to standard recommendation system design principles encouraging exploration across items and diversity of recommendations. Nonetheless, this result fleshes out the design space, and is given formally in Corollary \ref{cor:deterministc_pit_pil}.

\begin{corollary}\label{cor:deterministc_pit_pil}
    When $\pi_t$ is deterministic, setting $\pi_l \coloneqq \pi_t$ will minimize variance.
\end{corollary}

\begin{proof} 
See Online Appendix \ref{app:proof_deterministc_pit_pil}.
\end{proof}


However, when $\pi_t$ is \textit{not} degenerate, the optimal logging policy for minimizing worst-case MSE is not $\pi_t$ itself. 

\begin{corollary}\label{cor:on_policy_off_policy}
    Denote the empirical \textbf{on-policy estimator} as $\hat{V}_{\text{emp}}(\pi_t) = \frac{1}{N}\sum_{i = 1}^{N}r_i$ with actions drawn according to $\pi_t$. Then $\mathrm{MSE}[\hat{V}_{\text{IPW}}(\pi_t|\pi_l^*)] \leq \mathrm{MSE}[\hat{V}_{\text{emp}}(\pi_t)]$.
\end{corollary}

\begin{proof} 
See Online Appendix \ref{app:proof_on_policy_off_policy}.
\end{proof}

The optimal design $\pi^*_l$ in Proposition \ref{prop:known_mu_known_pit} has a notable side-benefit: it also weakly improves live reward accrual during the logging period, relative to running $\pi_t$ directly.

\begin{corollary}\label{cor:safety_off_policy}
     The policy value (expected accrued reward) under the optimal logging policy $V(\pi_l^*) \geq V(\pi_t)$ for all valid $\pi_t$ and for all $\mu(a, x) \in [0,1]$.
\end{corollary}

With knowledge of $\pi_t$ and $\mu$, Corollaries \ref{cor:on_policy_off_policy} and \ref{cor:safety_off_policy} form an important joint result: This general logging policy attain a (weakly) greater value than the target policy it is derived from, \emph{and} produces a more precise estimate of that target policy's true value than running it.

While emerging research on on-policy evaluation in bandit-adaptive designs highlights a trade-off between statistical inference and regret minimization \citep{duan2024regret, qin_optimizing_2024, simchi2023multi}, our result here demonstrates that it is possible, in certain settings, to jointly improve statistical precision and accrued regret under non-adaptive designs for off-policy evaluation. The implication for the designer of the logging policy is clear. To evaluate a single target policy, the designer should collect data under this modified solution $\pi_l^*$ and perform OPE of $\pi_t$ on this data, rather than live testing $\pi_t$.

\section{Logging Policy Design in Intermediate Information Settings}


%

In practice, the designer may not know the exact target policy they will be evaluating at logging time. A firm may anticipate evaluating several candidate policies without knowing in advance which will ultimately be deployed. For instance, it is likely the case that the currently deployed (``incumbent", ``production'') policy is already quite good, and that candidate policies will therefore be derived from it. This effectively constrains the class of target policies. Such a constraint may take the form of a lower bound on policy value, membership in a given policy class, or an organizational requirement such as a maximum deviation from the incumbent policy. In this case, the target policy can be characterized as a draw from a distribution. 

Furthermore, the reward function $\mu$ may be only partially known. In applications, we only have an estimate of $\mu$, for example from a machine learning model. This estimate carries uncertainty, which affects the ability to efficiently allocate sampling mass towards what to evaluate (cf. Propositions \ref{prop:known_mu_unknown_pit} and \ref{prop:known_mu_known_pit}).

Here the designer faces an \textit{intermediate} information setting. For analysis, we characterize these settings by two main features. The first is that a distribution over the target policies $\Pi_t$ is known. This framing reflects a setting where the designer does not know the exact target policy, but has prior knowledge about the class of policies that might be evaluated. The second feature is that a noisy estimate of $\mu$, the conditional expectation of rewards, is available at the design stage. For example, the designer has access to estimates of reward probabilities for the available items based on item features and user contexts. In recommendation systems, such an estimate is generally available from prior logged data, either via empirical frequencies after a cold-start period, or from a statistical model that generalizes over contexts. For the simulation results below, we make the further assumption that $\Pi_t$ is restricted to some class of ``reasonable'' policies.

\subsection{Known $\Pr(\pi_t)$}


We start with the case where the designer does not know the exact target policy but can formulate a prior distribution over the policy class. In this setting, a firm may wish to derive a single logging policy for the evaluation of multiple known target policies, linearly weighting the MSE of each in the overall objective. 

Let $\Pi_t$ be the \textit{distribution} that $\pi_t$ can assume.\footnote{This is a minor overloading of $\Pi_t$, so as not to introduce too many symbols.} Without a single $\pi_t$, we now have multiple possible objectives to minimize when designing a logging policy. Our objective is to minimize expected MSE, averaging over the distribution of target policies $\Pi_t$. Formally, the objective is: 
\[
\min_{\pi_l} \mathbb{E}_{\pi_t \sim \Pi_t} [\mathrm{MSE}(\pi_t|\pi_l)].
\]
From this, we can solve for the optimal logging policy. 
\begin{proposition}\label{prop:expected_pit}
    With known $\Pr(\pi_t)$, define $\tilde{\pi}_t(a|x)\coloneqq\sqrt{\mathbb{E}[\pi_t(a|x)^2]}$. Under the overlap assumption, $\pi_l^* =\frac{\tilde{\pi}_t(a|x)\sqrt{\mu(a, x)}}{\sum_{a'\in\mathcal A}\tilde{\pi}_t(a|x)\sqrt{\mu(a',x)}}$ satisfies $\min_{\pi_t} \mathbb{E}_{\pi_t \sim \Pi_t}[\mathrm{MSE}(\pi_t|\pi_l)]$.
\end{proposition}

This result follows directly from the structure of the expected MSE and allows for the simple construction of a \emph{plug-in pseudo-target policy} in the logging policies introduced above. Intuitively, uncertainty over the target policy is summarized by $\tilde{\pi}_t(a|x)$, which can be treated as a pseudo target when allocating action probabilities at logging time. We next derive a corresponding plug-in construction for $\mu$ in settings where reward estimates are noisy.


\subsection{Noisy $\mu$}

For applications like recommender systems, the logging policy designer may only have partial knowledge of the reward function $\mu$. In these cases, the designer will typically have a model-based estimate $\hat{\mu}_l$ learned from data previously collected. For instance, one may use a machine learning model that predicts the probability of a user streaming a recommendation conditional on it being impressed and given contextual features of items and the user. Even if the training dataset is large and has broad coverage, $\hat{\mu}_l$ will still tend to deviate from the true $\mu$. Likewise, the target policy will often be derived from noisy (but potentially different) model estimates $\hat{\mu}_t$. 

In Figure \ref{fig:mse_in_noise}, we show how using $\hat{\mu}_l$ as a surrogate for $\mu$ can affect the resulting MSE for OPE. We simulate the performance of logging policies under uncertainty in the reward function $\mu$ by computing the MSE of target policy value estimates across a range of noise levels in reward predictions. In the simulations, we draw $\mu$ according to a geometric decay in reward probabilities across actions, as described in Online Appendix~\ref{app:simulations}. We produce noisy reward predictions as $\hat{\mu}_l(a,x) = \epsilon(a,x) \mu(a,x)$, where $\epsilon(a,x) \sim \mathcal{N}(1,\sigma)$ adds multiplicative noise and heteroskedasticity in the reward predictor. To generate a target policy $\pi_t$, we analogously construct $\hat{\mu}_t(a,x) = \epsilon(a,x)\,\mu(a,x)$.\footnote{The standard deviation of this noise is chosen empirically to be .25 to ensure a ``good'' (see the figure) but not perfect target policy.} We then define the target policy to evaluate, $\pi_t$, to place equal probability mass on the top 30 actions ranked by $\hat{\mu}_t(\cdot\mid x)$ for each context $x$. This target policy reflects the usual desire in recommendation systems of balancing immediate reward accrual with exploration and diversity of recommendations, and is in the spirit of the top-$k$ policy we consider in the next subsection. Note that a fully greedy (top-1) target policy would instead coincide with the optimal logging policy, as discussed earlier. Full simulation details are provided in Online Appendix~\ref{app:simulations}.

Figure \ref{fig:mse_in_noise} shows the MSE for estimating the value of $\pi_t$ from various logging policies, as a function of the level of noise $\sigma$ in the reward predictor $\hat{\mu}_l$. We observe that the MSE of off-policy evaluation under this logging policy varies with the precision of $\hat{\mu}_l$, and may be lower, equal to, or higher than that of on-policy evaluation. Still, some excess error relative to on-policy evaluation is generally tolerable, as off-policy evaluation is often preferred or necessary due to practical constraints. Nonetheless, we show in the next subsection that the designer can mitigate this risk of off-policy evaluation by shrinking the reward predictions toward a per-context mean, reducing the effective noise in the off-policy evaluations from optimized logging policy designs compared to plugging-in reward estimates directly.

\begin{figure}[ht]
\FIGURE
    {
    \centering
    \includegraphics[width=9cm]{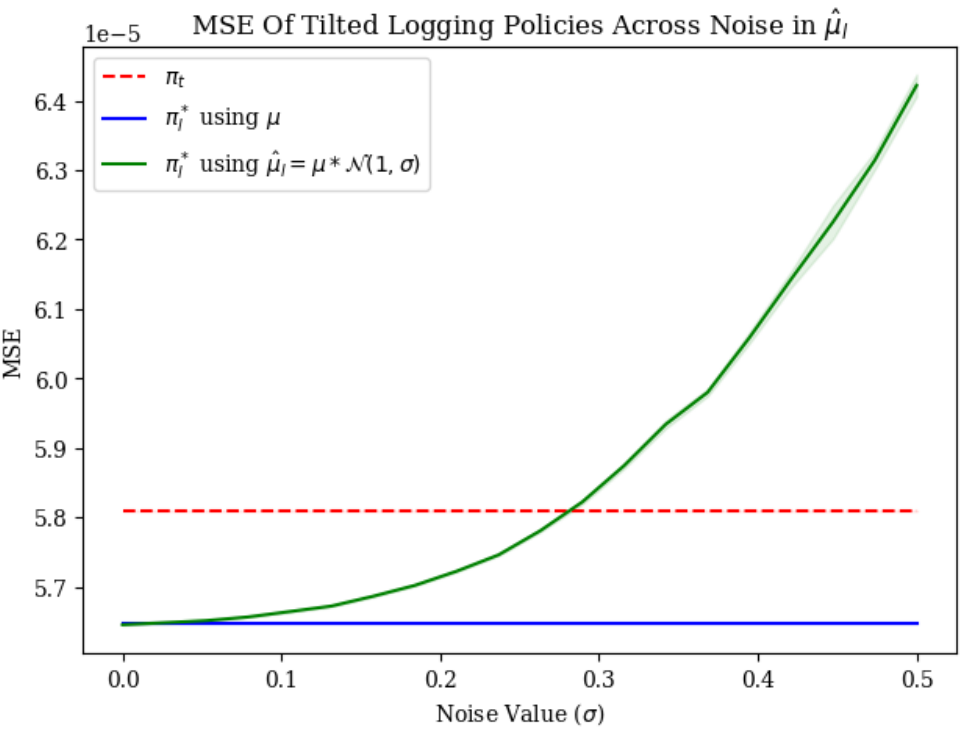}
    }
    {
    MSE as a function of the level of noise in the reward estimates $\hat{\mu}$.
    \label{fig:mse_in_noise}
    }
    {
    The (red) dashed line shows the MSE of on-policy estimation. Estimation with the ``optimal" logging policy can be either better or worse than on-policy estimation, depending on how noisy the reward estimates are, with a convex shape of the expected error in the noise in reward estimates.
    }
\end{figure}


\subsection{Posterior Shrinkage}

When the reward function $\mu$ is not known exactly, it is often suboptimal to use $\hat{\mu}_l$ as a plug-in estimator for designing logging policies. Intuitively, the uncertainty in $\mu$ increases as the error in $\hat{\mu}_l$ grows. In the extreme where noise dominates variation in estimates, $\hat{\mu}_l$ carries little information about $\mu$. Allocating probability mass toward such estimates could produce logging policies that perform worse than uniform randomization.

This observation motivates using some degree of \textit{posterior shrinkage}: we want our reward predictor to shrink towards a per-context \emph{average} the more uncertain its reward prediction is for a given context-action pair. This form of shrinkage is closely related to empirical Bayes estimators under hierarchical models, where noisy conditional point estimates are pulled to a common mean to reduce the MSE of an estimator \citep{robbins1956empirical, efron1973stein}. We now derive a practical procedure for implementing posterior shrinkage in logging policy designs for off-policy evaluation via the IPW estimator. This enables the designer to regularize noisy reward estimates toward more stable logging policies.

\subsubsection*{Theoretical derivation of shrinkage weight}

Optimal shrinkage varies in how reward probabilities are distributed per context and how noise is distributed in reward estimates. We make standard assumptions on each to make this problem tractable. To arrive at a simple and interpretable shrinkage rule, we adopt the following hierarchical Gaussian prior. First, we assume that the reward predictor is Gaussian distributed about the underlying reward function with fixed variance $\varepsilon^2$ representing noise:
\[
\hat{\mu}(a,x) \sim \mathcal{N}\left(\mu(a,x), \varepsilon^2 \right).
\]
We then assume that the reward function $\mu(a,x)$ itself is normally distributed about the per-context, across-action mean reward $m(x)$ with variance $\sigma^2$:
\[
\mu(a,x) \sim \mathcal{N}\left(m(x), \sigma^2 \right).
\]
This yields the posterior
\[
\mu(a,x)|\hat{\mu}(a,x) \sim \mathcal{N}\left( \frac{\frac{\hat{\mu}(a,x)}{\varepsilon^2}+\frac{m(x)}{\sigma^2}}{\frac{1}{\varepsilon^2}+\frac{1}{\sigma^2}}, \frac{1}{\frac{1}{\varepsilon^2}+ \frac{1}{\sigma^2}}\right).
\]
Rewriting, we get the posterior mean
\[
\mathbb{E}[\mu(a,x)|\hat{\mu}(a,x)]=\frac{\sigma^2\hat{\mu}(a,x)+\varepsilon^2m(x)}{\varepsilon^2 + \sigma^2}.
\]
Hence, given the reward predictor, the expected value of the underlying reward function is a linear combination of the predictor $\hat{\mu}$ and the true per-context expectation $m(x)$. 

Letting $\tilde{\mu}\coloneqq\mathbb{E}[\mu(a,x)|\hat{\mu}(a,x)]$, we have the form $\tilde{\mu}(a,x)=(1-w)\hat{\mu}(a,x)+wm(x)$ for some shrinkage weight $w\in[0,1]$ based on the ratio of the estimation variance $\varepsilon^2$ and cross-action variability $\sigma^2$; that is, $w = \varepsilon^2/(\varepsilon^2+\sigma^2)$.

Figure~\ref{fig:shrinkage_mse_weight} illustrates the effect of using the posterior shrinkage predictor $\tilde{\mu}$ for different values of the shrinkage weight $w$. We see that when noise is present in $\hat{\mu}$, moderate shrinkage toward the per-context mean reward reduces MSE for estimating the target policy value, and may even produce estimates of the target policy's value that are more accruate than running the target policy and calculating empirical reward (i.e., on-policy evaluation) .\footnote{Simulation parameters are presented in Online Appendix \ref{app:simulations}.} As the shrinkage weight approaches $1$, the shrinkage effect cancels out across actions and the correction toward the mean reward disappears.

However, the shrinkage weight of the policy must be applied \textit{ex ante}, prior to data collection under the logging policy. As a result, empirical grid search over values of $w$ is infeasible in practice. We therefore require a principled method for selecting the shrinkage weight based on information available at the time of logging policy design. We now discuss how to do so using evaluation-set level metrics on $\hat{\mu}$. 

With full information about $\mu$ and $\hat{\mu}$, the optimal shrinkage can be solved analytically. In practice, however, we observe only the realized reward $r$ for context-action pairs in prior data, which affects our accuracy of $\hat{\mu}$ and knowledge of $\mu$ across \emph{all} action-context pairs. This reflects the bandit feedback setting and the fundamental problem of causal inference.

To that end, we use proxies $m(x) \approx \bar{\mu}(x)$ and $\varepsilon^2 + \sigma^2 \approx \text{Var}(\hat{\mu}(\cdot\mid x))$ to establish the expectation of the prior of $\mu$ the the total variance of $\hat{\mu}$. We now use logged data to estimate the quality of the estimate function $\hat{\mu}$ with metrics that serve directly as plug-in estimators for computing $w^*$. To estimate $\sigma^2$ directly, we can use observed rewards from any previously collected data, sampled under any logging policy, for which we also have the reward predictions on the same context-action pairs. We then use the following identity, which holds exactly under the mean-zero noise assumption on $\hat \mu$, to approximate the variance of $\mu$:
\[
\text{Cov}(\hat{\mu}(X,A), R) =  \mathbb{E}[\hat{\mu}(X,A) \cdot R] - \mathbb{E}[\hat{\mu}(X,A)]\mathbb{E}[R].
\]
Using the decomposition,
\begin{align}
\mathbb{E}[\hat{\mu}(X,A) \cdot R] 
&= \mathbb{E}[\mathbb{E}\{\hat{\mu}(X,A)\cdot R \mid {\mu}(X,A)\}]\\
 &=\mathbb{ E}[\mathbb{E}\{\hat{\mu}(X,A) \mid {\mu}(X,A)\}]  \cdot  \mathbb{E}[R \mid {\mu}(X,A)]\\
 &= \mathbb{E}[{\mu}(X,A) \cdot {\mu}(X,A)] \\
 &= \mathbb{E}[{\mu}(X,A)^2],
\end{align}
we arrive at the form
\[
\text{Cov}(\hat{\mu}(X,A), R) =  \mathbb{E}[{\mu}(X,A)^2] - \mathbb{E}[\hat{\mu}(X,A)]\mathbb{E}[R].
\]
By the mean-zero noise assumption on $\hat{\mu}$ and with $R$ drawn from i.i.d. Bernoulli trials, we then have
\begin{align}
\text{Cov}(\hat{\mu}(X,A),R)& =  \mathbb{E}[{\mu}(X,A)^2] - \mathbb{E}[{\mu}(X,A)]^2\\
& = \text{Var}[{\mu}(X,A)]\\
&= \sigma^2.
\end{align}
Specifically, we leverage an auxiliary dataset $\{(x_i,a_i,r_i)\}_{i=1}^M$ of realized context, action, reward triples, over which $\hat{\mu}(x_i,a_i)$ is defined for all $i \in \{1,\ldots,M\}$. So, we can approximate the optimal shrinkage weight $w^*$ empirically as
\begin{align}\label{eq:w_opt}
w^* &= 1-\frac{\text{Cov}[\hat{\mu}(X,A),R]}{\text{Var}[\hat{\mu}(X,A)]}\\
&\approx 1-\frac{\widehat{\text{Cov}}_M[\hat{\mu}(x_i,a_i),r_i]}{\widehat{\text{Var}}_M[\hat{\mu}(x_i,a_i)}
\end{align}
where the hat and sample size subscripts denote that these are the empirical covariance and variance. This yields the posterior-shrinkage predictor of rewards  as
\[
\tilde{\mu}^*(a|x) = (1-w^*)\hat{\mu}(a|x) + w^*\bar{\mu}(x).
\]
The data used for computing Equation~\eqref{eq:w_opt} should be held out from the data used to train $\hat \mu$ to avoid optimistic bias in the estimate of $w^*$. Otherwise, the fit of $\hat \mu$ will inflate the empirical covariance in the numerator, biasing $w^*$ downward. 

Figure \ref{fig:shrinkage_noise} demonstrates the efficacy of this approach through simulation, using simulation parameters similar to those described above, which are described in detail in Online Appendix \ref{app:simulations}. We see in that even when the underlying distribution of $\mu$ and $\hat{\mu}$ is not Gaussian, this weight selection rule is still effective and can still yield MSE that is better than on-policy evaluation.

\begin{figure}[ht]
\FIGURE
    {
    \centering
    \begin{subfigure}[b]{0.425\textwidth}
    \centering
    \includegraphics[width=\linewidth]{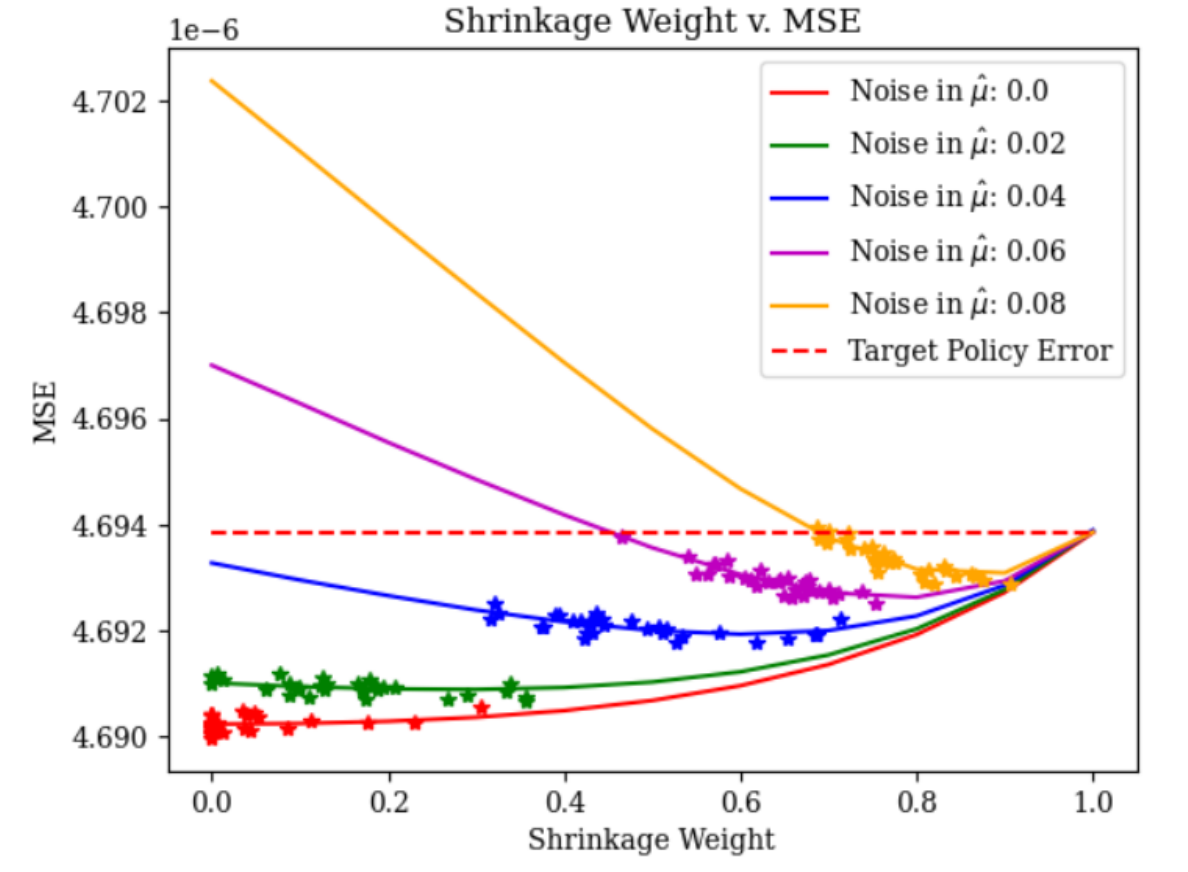}
    \caption{
    MSE as a function of shrinkage weight
    }
    \label{fig:shrinkage_mse_weight}
  \end{subfigure}
  \hfill  
  \begin{subfigure}[b]{0.485\textwidth}
    \centering
    \includegraphics[width=\linewidth]{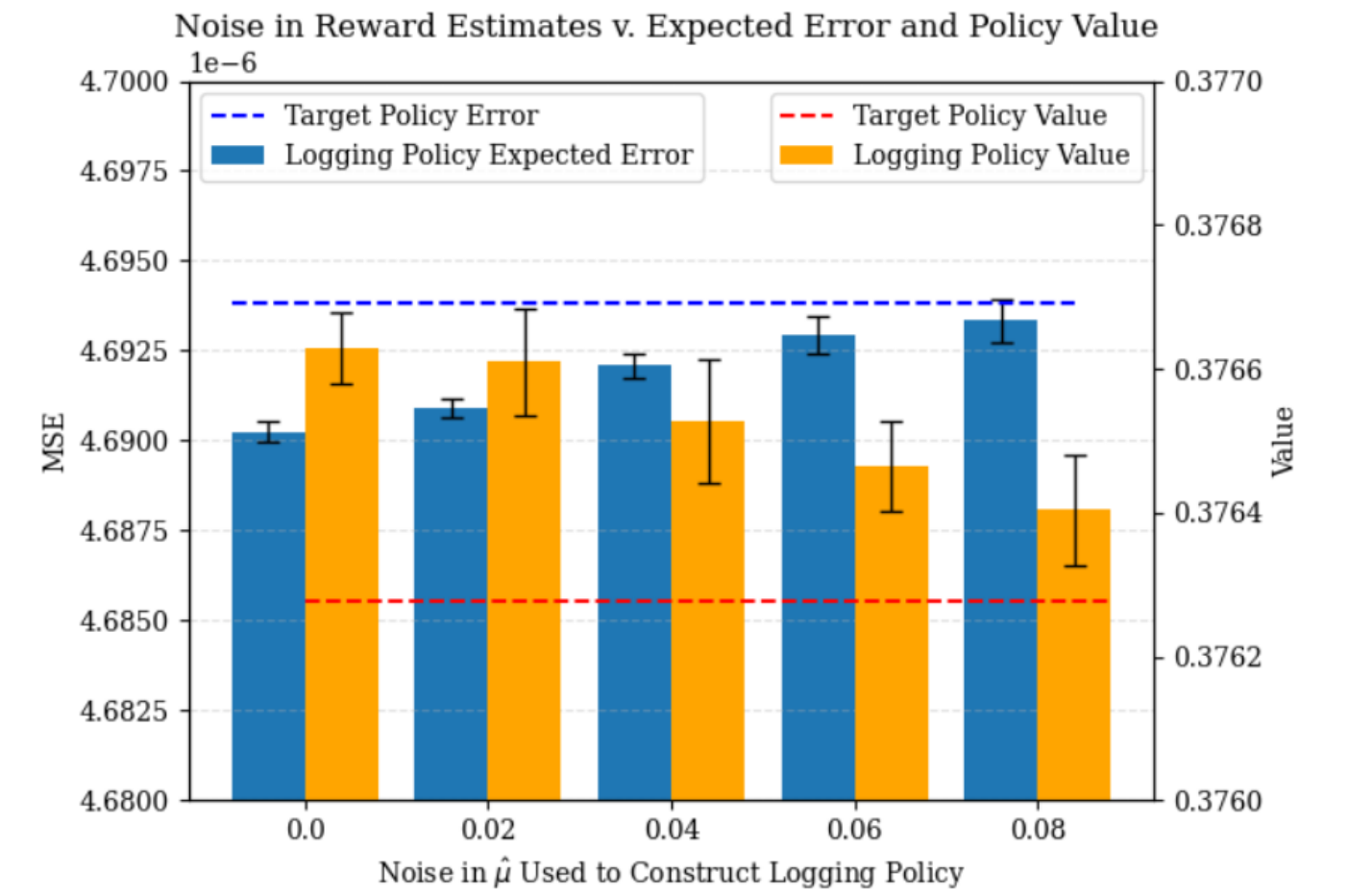}
    \caption{
    MSE and policy value under shrinkage
    }
    \label{fig:shrinkage_noise}
  \end{subfigure}
    }
    {
    Effect of posterior shrinkage and reward prediction noise on MSE and policy value
    \label{fig:shrinkage}
    }
    {
    A demonstration of $\text{MSE}(\pi_t|\pi_{l,\tilde{\mu}}^*)$ and $V(\pi_{l,\tilde{\mu}}^*))$ with $\pi_{l,\tilde{\mu}}^*$ constructed according to various shrinkage values and noise values in estimates. Shrinkage weights $w^*$ are constructed according to our approximate solution.
    Panel (a) plots the expected MSE of the IPW estimator for estimating $V(\pi_t)$ as a function of the shrinkage weight $w$, across different noise levels in the reward predictor $\hat{\mu}$. Stars represent the shrinkage weight computed according to the derived shrinkage rule and associated MSE of the derived optimal policy across each of $30$ trials. ``Target Policy Error'' represents the MSE of on-policy value estimator, obtained as the sample average observed reward on data from the target policy. 
    Panel (b) reports the MSE of $\hat{V}_{\mathrm{IPW}}(\pi_t\mid\pi_{l,\tilde{\mu}})$ and the logging policy value $V(\pi_{l,\tilde{\mu}})$ when using the reward predictor $\tilde{\mu}^*$ constructed according to the empirically optimal shrinkage weight $w^*$, compared to on-policy evaluation of $\pi_t$, and across noise in the estimates of $\hat{\mu}$ (30 trials per noise level). Even across levels of noise, the optimal logging policy yields greater expected reward during the logging period and lower MSE in estimating the target policy value than on-policy evaluation, both as byproducts of the MSE-minimizing design. Simulation details are provided in Online Appendix~\ref{app:simulations}.
    }
\end{figure}


  

\section{Discussion}\label{sec:Discussion}
In some settings, designers may have limited ability to specify policy classes. We now provide high-level guidance on how a designer can choose which logging policy to implement, viewed through the lens of the reward-coverage tradeoff described in Section \ref{sec:reward_coverage_tradeoff}.




\subsection{Restrictions on $\Pi_l$}\label{sec:restrictions}

In practical settings, there may be organizational or technical reasons why the designer is not able to deploy or justify deploying the exact optimal logging policy for minimizing MSE. For example, a fully greedy optimal logging policy may not be deployable due to the firm's desire to maintain diversity in recommendations. Alternatively, a logging policy that strays far from known high-reward actions may be deemed risky. To capture such constraints, we consider logging policies indexed by a single ``greediness' parameter that trades off mass on high-reward actions against broader coverage of actions the target policy may take. We next demonstrate how MSE varies with this greediness parameter across two logging policy classes.

These classes describe mappings from reward estimates to policies. Importantly, these classes can be used to define target policies in addition to logging policies. Their use as target policies describe some simple model of production recommendations with a dial balancing exploration and exploitation that could be made by the downstream, target policy designer. However, our analysis focuses on the efficacy of these classes for logging policy design.

We define a (fully) greedy policy to be one that deterministically selects the action per context $x$ with highest predicted reward; that is $A = \argmax_{a \in \mathcal A}\hat{\mu}_t(a, x)$ for all $x\in \mathcal X$. We begin by examining how the MSE from off-policy evaluation with the IPW estimator varies as a logging policy shifts from fully greedy to uniformly random. We refer to policies at intermediate values of this parameter as \textit{soft-greedy}. To assign this level of greediness to a scalar value, we define three related policy classes: \textit{top-$k$}, \textit{power-normalized} and \textit{softmax}. 

\subsubsection*{top-$k$ Policy Class}
The \textit{top-k} policy space $\Pi_{TK}(\hat{\mu})\coloneqq\Big\{ \pi_k \colon k \in 1, \ldots,|\mathcal{A}| \Big\}$, is the set of policies $\pi_k$  that assign equal weight to the $k \leq |\mathcal A|$ actions with the highest values of $\hat{\mu}(A,X)$ for each context $X=x$. Thus, $\pi_{k=1}$ is the fully greedy policy and $\pi_{k=|\mathcal A|}$ is the fully uniform policy.

\subsubsection*{Softmax Policy Class}
The \textit{softmax} policy space  smoothly adjusts for greediness: $\Pi_{SM}(\hat{\mu}) \coloneqq \Big\{\pi_\alpha, \pi_{\alpha}(a|x)=\frac{e^{\alpha\hat{\mu}(a,x)}}{\sum_{a' \in \mathcal{A}}e^{\alpha\hat{\mu}(a',x)}} \colon \forall \alpha \in [0,\infty)\Big\}$. Here, $\alpha = 0$ corresponds to the uniform policy while $\alpha = \infty$ equates to a greedy policy. Because this class of logging policies places probability mass on all actions for $\alpha \in [0,\infty)$, softmax policies ensure the IPW estimator $\hat V(\pi_t \mid \pi_l)$ is unbiased for any target policy $V(\pi_t)$ even if this mass is very small. This policy is directly related to Boltzmann sampling.

\subsubsection*{Power-normalized Policy Class}
The \textit{power-normalized} policy space also smoothly adjusts for greediness. The class is defined as  $\Pi_{PN}(\hat{\mu}) \coloneqq \Big\{\pi_d, \pi_{d}(a|x)=\frac{\hat{\mu}^d(a,x)}{\sum_{a' \in \mathcal{A}}\hat{\mu}^d(a',x)} \colon \forall \;d \in [0,\infty)]\Big\}$. Here, $d = 0$ corresponds to the uniform policy while $d= \infty$ equates to a greedy policy. This policy also enjoys the same property of ensuring weak overlap with any target policy, as long as reward estimates are bounded away from 0.

\subsection{Simulations Across Logging Policy Classes}

We simulate how the choice of greediness affects the MSE of IPW estimates using a simulated environment, detailed in Online Appendix \ref{app:simulations} . Here, $\pi_t$ is defined to be the \textit{top-200} policy based on $\hat{\mu}_t$.

Figure \ref{fig:softgreedy_errors} shows the results for small ($N=1000$) and large ($N=100,000$) sample-size settings, where the target policy $\pi_t$ is a top-$200$ policy based on $\hat{\mu}_t$ and a large action space of $|\mathcal{A}|=100,000$. We see that even these substantially constrained logging policy classes can achieve strong performance, approaching the MSE of the theoretical optimum, and in some cases matching or improving upon on-policy evaluation as a natural benchmark when the greediness parameter is chosen appropriately.

In both the small-sample and large-sample settings, well-tuned \textit{top-$k$} policies substantially outperform \textit{softmax} policies, with optimal values of \textit{k} around $200$. This result demonstrates that narrowing the logging support toward high-reward actions---intentionally incurring some bias---can reduce overall MSE relative to policies that enforce overlap (i.e., \textit{softmax} and\textit{ power-normalized}). As the sample size increases, the policies that enforce overlap perform relatively better. This is driven by the fact that bias does not shrink in sample size, but variance does. This phenomenon is also what leads the optimal $k$ value to be larger in the large sample-size setting. Likewise, power-normalized policies also perform well in the large-sample regime. With appropriate tuning, they approach the performance of top-$k$ while controlling worst-case MSE over a reasonable range $d$. Yet, for any finite sample size, as $d$ and $\alpha$ approach $\infty$ (i.e. approaching greedy), MSE approaches infinity. The strong performance of top-$k$ policies and bounded error near greedy suggest that truncating support can improve MSE when logging propensities are small relative to sample size. This findings is consistent with the sufficiency condition in Proposition \ref{prop:sampling_sufficiency}.

More generally, these results show that simple soft-greedy approaches to logging policy design can yield substantial improvements in off-policy evaluation relative to uniform logging, although their performance benefit depends heavily on the choice of greediness parameter. Moreover, the inability for the biased top-$k$ policy to reduce error relative to the variance-optimal policy described in Proposition $\ref{prop:known_mu_known_pit}$ suggests that restricting the optimal policy derivation to variance minimization does not sacrifice the performance gains available from intentionally incurring bias. We examine further heuristic variants in Online Appendix \ref{app:simulations}.

\begin{figure}[h!]
\FIGURE
    {
    \centering
    \begin{subfigure}[b]{0.475\textwidth}
        \includegraphics[width=\linewidth]{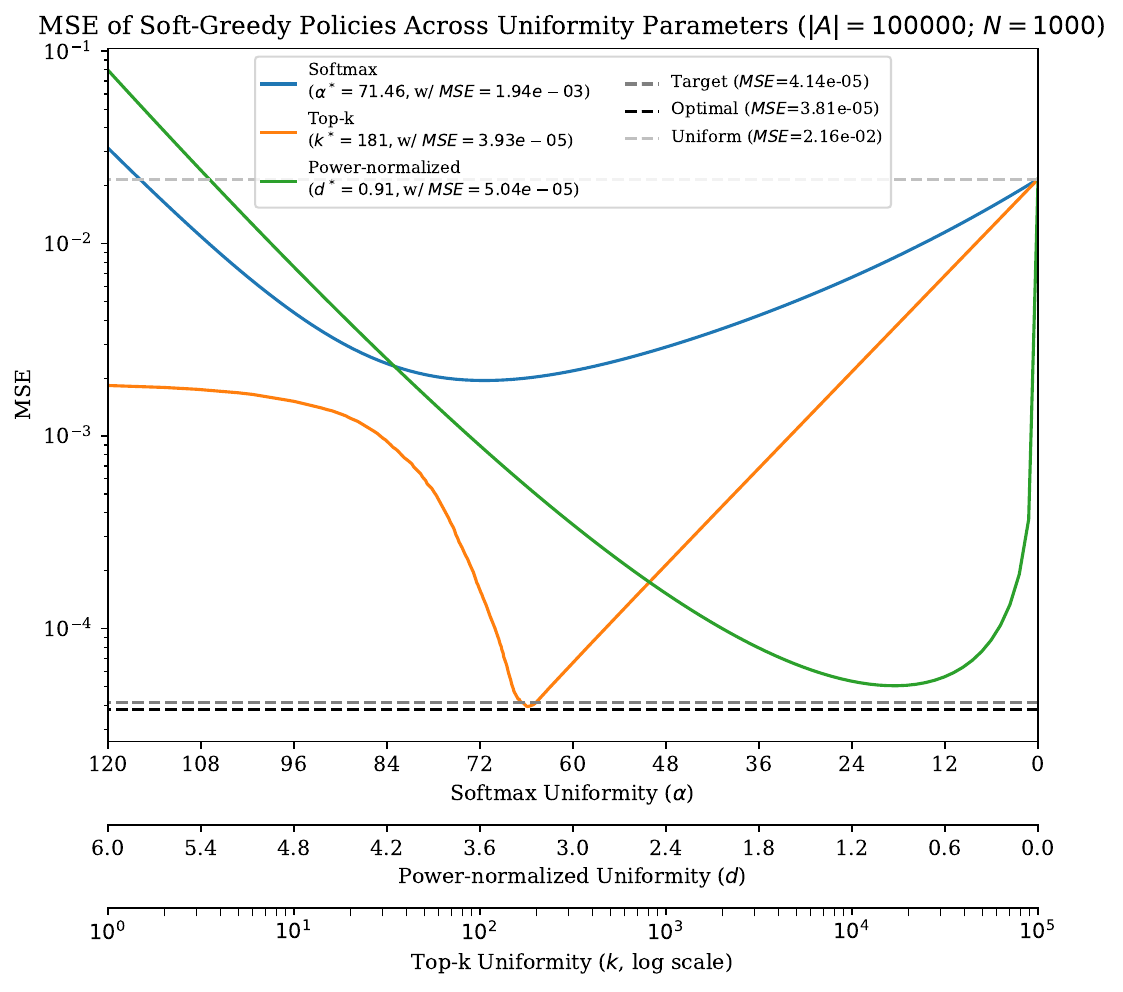}
        \caption{\textit{Small action space} $|\mathcal{A}|=1,000$.} 
    \end{subfigure}
    \hfill
    \begin{subfigure}[b]{0.475\textwidth}
        \includegraphics[width=\linewidth]{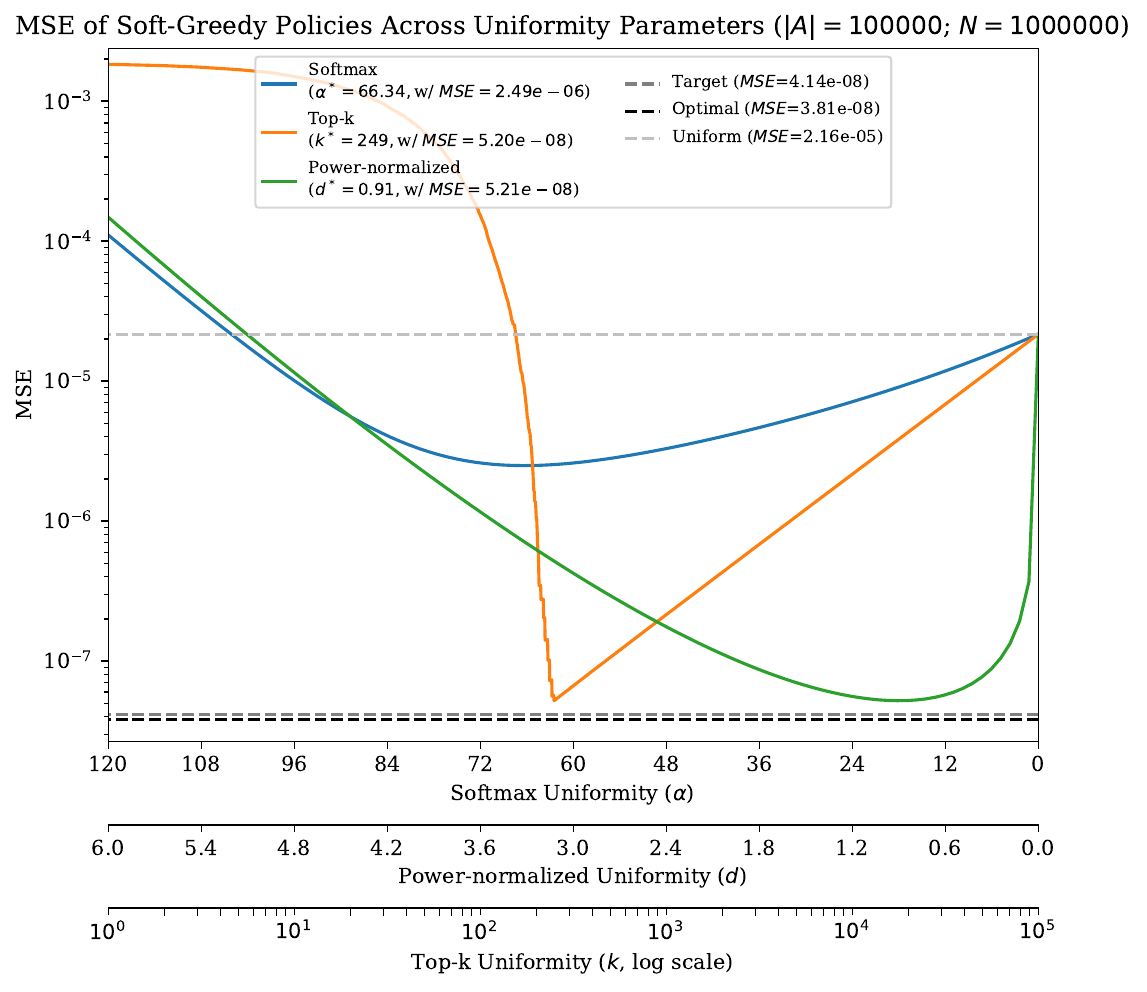}
        \caption{\textit{Large action space} $|\mathcal{A}|=100,000$.}
    \end{subfigure}
    }
    {
    MSE of IPW estimator for soft-greedy logging policy classes \textit{top-$k$} and \textit{softmax}. 
    \label{fig:softgreedy_errors}
    }
    {
     Simulation evidence of the MSE of the IPW estimator as a function of the greediness parameter for the \textit{top-$k$}, \textit{softmax}, and \textit{power-normalized} logging policy classes. The left edge of the \textit{x axis} of each panel corresponds to (near) greedy logging, while the right edge corresponds to uniform random logging. The MSE is measured with respect to the value (expected reward) of a \textit{top-$k$} ($k=200$) target policy based $\mu$. "Target" shows the MSE of running the target policy, if it were known.
     Panel (a) shows that the MSE-minimizing top-$k$ logging policy occurs at $k^* = 181$, yielding $\text{MSE}=3.93 \times 10^{-5}$; 
     the MSE-minimizing softmax parameter is $\alpha^* = 71.46$, yielding $\text{MSE}=1.94 \times 10^{-3}$; and the SE-minimizing power-normalized degree is $d^* = .91$, yielding $\text{MSE}=5.04 \times 10^{-5}$.
     Panel (b) shows analogous results for the large sample-size setting, with MSE-minimizing top-$k$ being $k^* = 249$ for the top-$k$ logging policy, yielding $\text{MSE}=5.20 \times 10^{-8}$; the MSE-minimizing $\alpha$ value for softmax is $\alpha^* = 66.34$, yielding $\text{MSE}=2.49 \times 10^{-6}$; and the MSE-minimizing degree $d$ for softmax is $d^* = .91$, yielding $\text{MSE}=5.21\times 10^{-8}$.
     Full simulation details are in Online Appendix \ref{app:simulations} 
    }
\end{figure}

Selecting the greediness parameter in practice can be guided by the simulation results above. When the target policy class is known, $k$ can be initialized at the target policy's approximate support size, as the MSE-minimizing $k$ closely tracks this quantity across both sample-size regimes. For power-normalized policies, $d$ can be tuned analogously via held-out MSE evaluation. For softmax policies, $\alpha$ lacks a direct analytical interpretation and should be selected via held-out MSE evaluation; however, given the consistently weaker performance of softmax relative to top-$k$ and power-normalized across our simulations, the latter two classes are generally preferable when the action space is large.

When the target policy is unknown, a moderately soft-greedy policy is preferable to a fully greedy one, since bias from support truncation does not diminish in sample size. In all of these settings, a natural approach is to treat the parameter as a hyperparameter and select it by computing IPW estimates of a reference policy's value on historically collected data, choosing the value that minimizes the resulting MSE. This requires only prior logged data and reward predictions --- both available before the logging period begins.


\section{Related work}

The closest body of related work is the nascent stream examining the design of safe logging policies. In the most closely related work, \citet{Wan2022Safe} describe the design of safe logging policies to evaluate bandit policies offline. Here, safety is defined as the value of the logging policy (average reward accrual) being no less than some fixed proportion of the value of a baseline policy. The authors highlight standard approaches, involving a mixture between some given policy and a uniform exploration logging policy. Under hard safety and sample budget constraints, the authors derive optimal logging policies for various bandit agents to best study the difference in policy values relative to some baseline policy. 

Other works study the augmentation of logging policies for better OPE. \citet{tucker2023mval} is most similar in scope to ours, but describes a method for adding observations. The resultant method confirms the intuition behind our solution to the fully observable problem. Demonstrating the sub-optimality of an $\epsilon$-mixing approach, \citet{zhu_safe_2022} cast logging policy design as an optimal experimental design problem. These authors design an optimal solution derived from \textit{G-optimal design} which describes the quality of a logging policy in terms of information gain. Their optimal policy follows a ``water-filling'' method, ensuring the minimal action-context pairs have equal propensity scores. However, this derived policy is tailored to off-policy learning, not to off-policy evaluation.

\citet{Li2023Optimal} study optimal treatment allocation for OPE in sequential decision-making problems. The authors here propose allocating more samples to actions where IPW estimates would be noisy, due to the target policy or due to reward variance. Our approach follows a similar thread, where we highlight that a \textit{reward-coverage tradeoff} needs to be balanced in designing a logging policy; however, we note that mass needs to be allocated explicitly toward high-reward items, not high-variance-in-reward items. 

An emerging literature focuses on developing estimators for use in experimental data under various assumptions. \citet{sakhi2025practical} describe a more efficient estimator of the ATE of treatment policies when run via A/B tests. This estimator leverages overlap in the treatment policies being tested, a key factor we consider when designing logging policies. \citet{li15toward} consider the worst case estimation error in a similar setup to one of our ``informational settings.'' Their results focus on describing bounds on minimax squared error across several estimators; we focus on the optimal design of logging policies to achieve minimal squared error. \citet{Ma2022Minimax} study minimax optimality in a similar setting but propose a switch estimator to achieve an optimal competitive ratio on this risk. Their work looks at this risk under varying levels of knowledge about the logging policy, whereas our work designs the logging policy under varying levels of knowledge about the setting.

Another part of this literature focuses on how to conduct statistical inference on data that has been adaptively collected via a bandit algorithm \citep{zhang2020inference, shen2024doubly, zhang2021statistical, bibaut2025demystifying, bibaut2021post, chen2021statistical}. A closely related stream of work focuses specifically on how to conduct OPE in contextual bandits \citep{wang2017optimal, zhan2021off}. Our work differs in that we focus on how to collect data for a given OPE estimator. Because the logging policy is non-adaptive, the OPE estimator can be applied as usual while retaining its properties for statistical inference.

Other recent research focuses on the online problem of simultaneous regret minimization and statistical inference on treatment arms, and the trade-off between the two objectives \citep{duan2024regret, simchi2023multi}. We do not focus on regret minimization or online inference. Nonetheless, our result that uniform logging is worst-case optimal for OPE (yet can accrue arbitrarily poor rewards) hints at a broader trade-off between data collection performance and what can be learned from it.

Taken together, the above literature emphasizes either safety constraints, designing efficient estimators, or a trade-off between regret minimization and statistical inference. Our contribution instead centers on the design of logging policies to minimize policy-value estimation error under informational constraints. This design perspective connects our problem to two foundational literatures.

The general problem of choosing assignment probabilities to improve estimator precision traces back to optimal experimental design, beginning with \citet{wald1943efficient} and developed through geometric and information-based criteria \citep{elfving1952optimum, kiefer1959optimum}. For recent surveys through an optimization lens, see \citet[][Sections 1 and 2]{xiong2025automated} and \citet{zhao2024experimental}. In this classical literature, treatment probabilities are selected to minimize the variance of an estimator of the average treatment effect, typically under a parametric outcome model. In contrast, our work considers the design of logging policies for off-policy evaluation, with the objective of minimizing MSE of a policy value estimate based on inverse propensity weighting. As such, we do not aim to maximize precision of an average contrast in potential outcomes, but rather minimize the estimation error of the average potential outcome under a counterfactual treatment assignment mechanism. Our logging policy solutions can be viewed as contextual analogs of unequal-probability sampling for the Horvitz–Thompson estimator \citep{horvitz1952generalization}, where action probabilities are chosen to control MSE. From an importance-sampling perspective, our results reflect the trade-off between variance inflation due to small proposal probabilities and bias from truncating the support of the proposal distribution \citep[e.g.,][]{ionides2008truncated}.

A second connection is to empirical Bayes methods for statistical decision problems, initiated by \citet{robbins1956empirical, robbins1964empirical}. In the empirical Bayes framework, parameters are treated as draws from an unknown prior, estimated from data and then substituted into the Bayes decision rule. Under an MSE criterion, optimal solutions to this problem formulation yield shrinkage estimators that reduce risk relative to naive plug-in estimation \citep[see, e.g.,][]{efron1973stein}. In our setting, direct plug-in estimators inflate the variance of induced inverse propensity weights when reward parameters are estimated with noise, but this can be mitigated via posterior shrinkage applied to the reward model's estimates. As such, our solution admits an empirical Bayes interpretation that reduces MSE of the IPW estimator while improving reward accrual during logging.


\section{Conclusion and Future Work}

This paper presents a unified framework for the optimal design of logging policies for off-policy evaluation. The central insight is that the logging policy is not merely a data collection mechanism. Instead, it is a design choice with direct consequences for the statistical quality of downstream policy evaluation. To this end, we frame logging policy design as the problem of minimizing the MSE of the IPW estimator, which is a simple and model-free estimator. Our analysis shows how firms can make principled, theoretically grounded choices about how to collect data for OPE, rather than relying on ad hoc heuristics, safety-only constraints, or uniform randomization as a conservative fallback.

Beyond this takeaway for practice, our work makes several contributions to the literature. First, it provides a clean sufficiency condition on when an action should be covered by the logging policy, while noting that there is no simple closed-form expression for when an action should be included. We describe a fundamental \textit{reward-coverage} tradeoff that needs to be balanced when designing a logging policy for use in OPE. We then provide optimal solutions to the problem of minimizing MSE of off-policy evaluation with the IPW estimator in settings of informational extremes. As part of our results, we note an important byproduct of the optimal logging policy: in certain settings, the optimal logging policy accrues higher expected reward than the target policy during the logging period, and yields an IPW estimate with lower MSE than the empirical on-policy estimator. The paper extends these findings in informational extremes to intermediate settings, and describes the construction of tractable logging policies for settings where reward probabilities are not exactly known. Finally, the paper provides more general insights into balancing the reward-coverage tradeoff when the theoretically optimal logging policy cannot be deployed. In particular, we provide approximate, soft-greedy methods for balancing these considerations and demonstrate that these can attain strong performance by balancing logging mass between high-reward items and coverage across items the target policy is likely to sample.

Taken together, our results provide a principled basis for logging policy design. The framework clarifies when and why different logging strategies are appropriate, and gives firms concrete tools for improving offline experimentation efforts within their operational and organizational constraints.

\paragraph{Limitations.} The results of this paper rely on several assumptions. The first assumption is that the IPW estimator is used for the off-policy evaluation. By being free of a reward model, this estimator decouples the design stage from the estimation stage. For model-based estimators such as the direct method or doubly robust, complications may arise when the outcome model used in estimation is correlated with the reward model used to design the logging policy. Extending our framework to such estimators is left for future work. Another limitation is that our theoretical results rely on weak overlap, leading our MSE minimization to reduce to variance minimization. In the sample limit, variance dominates bias, so the derived policies minimize MSE. In small-sample regimes, however, inducing bias can improve MSE relative to the variance-optimal policy, as we demonstrate through counterexample.

\paragraph{Future work.} There are several dimensions to explore further in future work. The first is a generalization of these results to arbitrary real-valued rewards. Similar extensions could expand on more sophisticated recommendation settings, such as slate-based or sequential recommendations. A second direction is optimal parameter selection for common soft-greedy policies.  Finally, exploring logging policy design experimentally could yield insights into the practical performance of our theoretically derived policies.

\printbibliography[title={References}]

\clearpage
\thispagestyle{empty}

\vfill
\begin{center}
  {\Huge\bfseries Online Appendix}
\end{center}
\vfill

\clearpage
\begin{APPENDICES}

\section{Additional Proofs}\label{app:proofs}

\subsection{Proof of Proposition \ref{prop:sampling_sufficiency}}
\label{app:sampling_sufficiency}
\begin{proof}
    Fix a context $x$ and set of actions $\mathcal{A}_x$ which receive support by $\pi_l(\cdot|x)$. Now, we show that the addition of some action $a' \notin \mathcal{A}_x$ to $\mathcal{A}_x$ will reduce MSE if $\pi_l(a'|x)>\frac{1}{\mu(a', x)(n\Pr(x')+1)}$.  If mass were placed on $a'$ meeting this criteria, we have the following derivation:
    \begin{align*}
        \pi_l(a'|x)&>\frac{1}{\mu(a', x')(n\Pr(x')+1)}\\
        \iff (N\Pr(x')+1)\mu(a', x')&>\frac{1}{\pi_l(a'|x')}\\
        \iff (N\Pr(x')+1)\pi_t(a'|x')^2\mu(a'|'x)^2&>\frac{\pi_t(a'|x')^2}{\pi_l(a'|x')}\mu(a', x')\\
        \iff N\Pr(x')\pi_t(a'|x')^2\mu(a', x')^2&>\frac{\pi_t(a'|x')^2}{\pi_l(a'|x')}\mu(a', x')-\pi_t(a'|x')^2\mu(a', x')^2\\
        \iff \pi_t(a'|x')^2\mu(a', x')^2&>\frac{1}{N\Pr(x')}\left[\frac{\pi_t(a'|x')^2}{\pi_l(a'|x')}\mu(a', x')-\pi_t(a'|x')^2\mu(a', x')^2\right].\\
        \iff \Pr(x')^2\pi_t(a'|x)^2\mu(a', x')^2&>\frac{\Pr(x')}{N}\left[\frac{\pi_t(a'|x')^2}{\pi_l(a'|x')}\mu(a', x')-\pi_t(a'|x')^2\mu(a', x')^2\right].
    \end{align*}

    We now manipulate the RHS to show that this quantity is at least the increase in the variance from sampling action $a'$ with probability $\pi_l>\frac{1}{\mu(a', x')(n\Pr(x')+1)}$. For notational simplicity, let $V(\mathcal{X} \setminus \{x'\})\coloneqq \frac{1}{N}\sum_{x \in \mathcal{X}\setminus \{x'\}}\Pr(x)\left[\sum_{a\in\mathcal{A}_{x}}\frac{\pi_t(a|x)^2}{\pi_l(a|x)}\mu(a, x')-\left(\sum_{a\in\mathcal{A}_{x'}}\pi_t(a|x)\mu(a, x')\right)^2\right]$ denote the variance term contributed by all contexts other than $x'$ which are added linearly. The RHS of the above inequality is then
       \begin{align*}
    \Delta\text{Variance}&\coloneqq\frac{\Pr(x')}{N}\left[\sum_{a\in\mathcal{A}_{x'}}\frac{\pi_t(a|x')^2}{\pi_l(a|x')}\mu(a, x')+\frac{\pi_t(a'|x')^2}{\pi_l(a'|x')}\mu(a', x')\right.\\
    &- \underbrace{\left.\left(\sum_{a\in\mathcal{A}_{x'}}\pi_t(a|x')\mu(a, x') + \pi_t(a'|x')\mu(a', x')\right)^2\right] +V(\mathcal{X}\setminus \{x'\})}_{\text{Variance of setting } \pi_l(a'|x')>0}\\
        &- \underbrace{\left(\frac{\Pr(x')}{N}\left[\sum_{a\in\mathcal{A}_{x'}}\frac{\pi_t(a|x')^2}{\pi_l(a|x')}\mu(a, x')-\left(\sum_{a\in\mathcal{A}_{x'}}\pi_t(a|x')\mu(a, x')\right)^2\right]+V(\mathcal{X}\setminus \{x'\})\right)}_{\text{Variance of setting } \pi_l(a'|x')=0 }\\
        &=\frac{\Pr(x')}{N}\left[\frac{\pi_t(a'|x')^2}{\pi_l(a'|x')}\mu(a', x') -\left(\sum_{a\in\mathcal{A}_{x'}}\pi_t(a|x')\mu(a, x') + \pi_t(a'|x')\mu(a', x')\right)^2 \right.\\
        &+ \left.\left(\sum_{a\in\mathcal{A}_{x'}}\pi_t(a|x')\mu(a, x')\right)^2\right]\\
         &\leq\frac{\Pr(x')}{N}\left[\frac{\pi_t(a'|x')^2}{\pi_l(a'|x')}\mu(a', x')-\pi_t(a'|x')^2\mu(a', x')^2\right]
    \end{align*}
    Likewise, we now manipulate the LHS to show that this term is no more than the increase in bias squared---that is, the other component of the MSE decomposition---from not sampling action $a'$. Similarly, let $B(\mathcal{X}\setminus \{x'\})$ correspond to the bias contributed by all contexts other than $x'$, with $B(\mathcal{X}\setminus \{x'\})\coloneqq\sum_{x \in \mathcal{X}\setminus \{x'\}}\Pr(x)\sum_{a\notin\mathcal{A}_x}\pi_t(a|x)\mu(a, x)$. The LHS of the inequality is then
    \begin{align*}
       \Delta\text{Bias}^2 &\coloneqq \underbrace{\left(B(\mathcal{X}\setminus \{x'\})+\Pr(x')\sum_{a\notin\mathcal{A}_x}\pi_t(a|x)\mu(a, x) + \Pr(x')\pi_t(a'|x)\mu(a', x)\right)^2}_{\text{Bias}^2\text{ of setting }\pi_l(a'|x)=0}\\
       -&  \underbrace{\left(B(\mathcal{X}\setminus \{x'\})+\Pr(x')\sum_{a\notin\mathcal{A}_x}\pi_t(a|x)\mu(a, x) \right)^2}_{\text{Bias}^2\text{ of setting }\pi_l(a'|x)>0}\\
       &\geq \Pr(x')^2\pi_t(a'|x)^2\mu(a', x)^2
    \end{align*}
    
    Now, we sandwich these inequalities to arrive at the per-context conclusion:
    \begin{align*}
        \Delta\text{Bias}^2\geq\Pr(x')^2\pi_t(a'|x)^2\mu(a', x)^2&>\frac{\Pr(x)}{N}\left[\frac{\pi_t(a'|x)^2}{\pi_l(a'|x)}\mu(a', x)-\pi_t(a'|x)^2\mu(a', x)^2\right]\geq\Delta\text{Variance}\\
        \Rightarrow \Delta\text{Bias}^2&>\Delta\text{Variance}.
    \end{align*}

\end{proof}

\subsection{Proof of Proposition \ref{prop:unknown_mu_unknown_pit}}\label{app:proof_uniform}
\begin{proof}
    Denote $\pi^{*}$ to be the logging policy solution to $\min\limits_{\pi_l \in \Pi} \max\limits_{\mu \in \mathcal{M}, \pi_t \in \Pi}\;\mathrm{MSE}[\hat{V}_{\text{IPW}}(\pi_t | \pi_l)]$. We first demonstrate that for any context $x$, $\pi^{*}(\cdot | x) = \pi_U(\cdot|x)$ minimizes the maximum MSE the target policy can induce. To begin, we define functions describing conditional error, that is, the MSE conditional on context, which constitutes one ``instance'' of the estimation problem. With this assumption that the logging policy has full support, we have no conditional bias. 
    
    Let $v(\pi_t,\pi_l \mid x)$ denote the contribution of context $x$ to the variance of $\hat{V}_{\text{IPW}}(\pi_t \mid \pi_l)$. Under full support, this equals the conditional variance given $X=x$. We have
    \begin{align*}
         \text{(Conditional IPW Variance)} \quad v(\pi_t,\pi_l \mid x) &\coloneqq\left(\sum_{a \in \mathcal{A}_x} \frac{\pi_t(a \mid x)^2}{\pi_l(a \mid x)} \mu(a, x)\right)-\left(\sum_{a \in \mathcal{A}_x} \pi_t(a \mid x)\mu(a, x)\right)^2.
    \end{align*}
   We now consider the support of $\pi_l$ on context $x$, $\mathcal{A}_x$, to be fixed. We show that with this fixed support, the maximum variance term $v$ is minimized for the designer by spreading probability equally across $\mathcal{A}_x$.
   
   \begin{lemma}[Conditional uniform across support]\label{lem:conditional}
       Fix support $\mathcal{A}_x$. $\pi^{*}_l(\cdot | x) = \frac{1}{|\mathcal{A}_x|} \forall {a \in \mathcal{A}_x}$ minimizes $\max\limits_{\mu \in \mathcal{M}, \pi_t \in \Pi}v(\pi_t, \pi_l | x)$.
    \end{lemma}
    
    \begin{proof}[Proof of Lemma \ref{lem:conditional}]
        Observe that $v_{IPW}(\pi_t, \pi_l | x)$ is convex in $\pi_t$. Thus, a deterministic $\pi_t(\cdot|x)$ will maximize $v_{IPW}(\pi_t, \pi_l | x)$, treating $\pi_l$ as fixed. Let $\pi^{*}_t$ be this deterministic policy with all its mass placed on action $a^*$. 

        We now have the function $g$ that the designer needs to minimize, with
        \begin{align}
            g(\pi_l | \mathcal{A}_x) = \max\limits_{a^* \in \mathcal{A}_x}\left[\frac{\mu(a^*|x)}{\pi_l(a^* | x)} - \mu(a^*|x)^2\right].
        \end{align}
        Because $\mu$ is not determined by the designer, solving $\pi_l^*(\cdot|x)$ reduces to minimizing $\max\limits_{a^* \in \mathcal{A}_x}\frac{\mu(a^*|x)}{\pi_l(a^* | x)} $. While, it follows that setting $\pi_l^*(\cdot|x)$ s.t. $\frac{\mu(a, x)}{\pi^*_l(a|x)} =\frac{\mu(a', x)}{\pi^*_l(a'|x)} \quad \forall \; a,a' \in \mathcal{A}_x $, without knowledge of $\mu$, all actions must be treated as having constant reward averages. Thus, $\pi_l^*(a|x)=\frac{1}{|\mathcal{A}_x|} \quad \forall \; a \in \mathcal{A}_x$ minimizes $\max\limits_{\mu \in \mathcal{M}, \pi_t \in \Pi}v_{IPW}(\pi_t, \pi_l | x)$.
    \end{proof}
    
Extending this argument to the aggregate level, we have in this setting with full support that 
\begin{align*}
    \mathrm{MSE}[\hat{V}_{\text{IPW}}(\pi_t | \pi_l)] &=\frac{1}{N}\left[
    \sum_{x \in \mathcal{X}} \Pr(x) \left(\left(\sum_{a \in \mathcal{A}_x} \frac{\pi_t(a \mid x)^2}{\pi_l(a \mid x)} \mu(a, x)\right)-\left(\sum_{a \in \mathcal{A}_x} \pi_t(a \mid x)\mu(a, x)\right)^2\right)\right]\\
    &=\frac{1}{N}\left[
    \sum_{x \in \mathcal{X}} \Pr(x) v_{\text{IPW}}(\pi_t,\pi_l\mid x)\right].
\end{align*}
Because $\Pr(x)$ and $n$ are treated as fixed and $v_{\text{IPW}}(\cdot|x)$ are combined linearly, $\pi_l(a|x) = \frac{1}{|\mathcal{A}|}$ minimizes $\max\limits_{\mu \in \mathcal{M}, \pi_t \in \Pi} \;\mathrm{MSE}[\hat{V}_{\text{IPW}}(\pi_t | \pi_l)]$.

\end{proof}

\subsection{Proof of Proposition \ref{prop:known_mu_unknown_pit}}
\label{app:proof_known_mu_unknown_pit}

\begin{proof}
    This proof follows a very similar structure to the proof of Proposition 1. We begin with showing the solution minimizing the conditional variance $v$.
    \begin{lemma}[Instance Level Uniform across support]\label{lem:instance-level-uniform}
       Fix support $\mathcal{A}_x$. Then $\pi_l^*(a|x) = \frac{\mu(a, x)}{c + \mu^2(a|x)}$ with $1 = \sum_{a \in \mathcal{A}_x}\frac{\mu(a, x)}{c + \mu^2(a|x)}$ minimizes $\max\limits_{\pi_t \in \Pi}v_{IPW}(\pi_t, \pi_l | x)$.
    \end{lemma}
    \begin{proof}[Proof of Lemma \ref{lem:instance-level-uniform}]
        As in Lemma 1, observe that $v_{IPW}(\pi_t, \pi_l | x)$ is convex in $\pi_t$. Thus, a deterministic $\pi_t(\cdot|x)$ will maximize $v(\pi_t, \pi_l | x)$, treating $\pi_l$ as fixed. Let $\pi^{*}_t$ be this deterministic policy with all its mass placed on action $a^*$. 

        Let $f(a|\pi_l)\coloneqq \frac{\mu(a, x)}{\pi_l(a | x)} - \mu(a, x)^2$. We now have the function $g$ that the designer needs to minimize, with
        \begin{align}
            g(\pi_l | \mathcal{A}_x) = \max\limits_{a^* \in \mathcal{A}_x}f(a^*|\pi_l).
        \end{align}
         We have that $f$ is decreasing and convex in $\pi_l$, so $g(\pi_l|\mathcal{A}_x)$ is minimized when $f(a|\pi_l) = f(a'|\pi_l) \; \forall \; a,a' \in \mathcal{A}_x$. By setting $f(a|\pi_l) = c \quad \forall \quad a\in\mathcal{A}$, we have $\pi_l^*(a|x) = \frac{\mu(a, x)}{c + \mu^2(a|x)}$ with $1 = \sum_{a \in \mathcal{A}_x}\frac{\mu(a, x)}{c + \mu^2(a|x)}$, noting here that $\mu(a, x) \leq 1, \pi_l(a|x)\leq 1 \quad \forall \quad a \in \mathcal{A} \implies c \geq 0$.
    \end{proof}
    
We now prove the proposition simply by extending this argument to the population-level marginal over contexts. For this setting with full support, we have that 
\begin{align*}
    \mathrm{MSE}[\hat{V}_{\text{IPW}}(\pi_t | \pi_l)] &=\frac{1}{N}\left[
    \sum_{x \in \mathcal{X}} \Pr(x) \left(\left(\sum_{a \in \mathcal{A}_x} \frac{\pi_t(a \mid x)^2}{\pi_l(a \mid x)} \mu(a, x)\right)-\left(\sum_{a \in \mathcal{A}_x} \pi_t(a \mid x)\mu(a, x)\right)^2\right)\right]\\
    &=\frac{1}{N}\left[
    \sum_{x \in \mathcal{X}} \Pr(x)v_{\text{IPW}}(\pi_t,\pi_l|x)\right].
\end{align*}
Because $\Pr(x)$ and $n$ are treated as fixed and $v_{\text{IPW}}(\cdot|x)$ are combined linearly, $\pi_l^*(a|x) = \frac{\mu(a, x)}{c + \mu^2(a|x)}$ minimizes $\max\limits_{\pi_t \in \Pi} \;\mathrm{MSE}[\hat{V}_{\text{IPW}}(\pi_t | \pi_l)]$.

\end{proof}

\subsection{Proof of Proposition \ref{prop:uknown_mu_known_pit}}
\label{app:proof_uknown_mu_known_pit}

\begin{proof}
    We examine the conditional variance by fixing $x$ and then extend this result to the population. We denote the value (mean reward) of $\pi_t$ on a given instance $x$ as $v(x) \coloneqq \sum_{a \in \mathcal{A}}\pi_t(a|x)\mu(a,x)$. We examine the worst-case MSE in cases where (1) $\pi_l\neq\pi_t$, and (2) $\pi_l = \pi_t$.
    
    We begin by defining and proving a useful intermediate lemma. 

    \begin{lemma}[Minimum sum]\label{lem:cauchy-schwarz}
        Let $C(\pi_l|x)\coloneqq\sum_{a \in \mathcal{A}}\tfrac{\pi_t(a|x)^2}{\pi_l(a|x)}$. Then, $C(\pi_l)=1$ iff $\pi_l(\cdot|x) = \pi_t(\cdot|x)$, otherwise $C(\pi_l|x)>1$.
    \end{lemma}
    \begin{proof}[Proof of Lemma \ref{lem:cauchy-schwarz}]
         Let $\vec{m}_a \coloneqq \frac{\pi_t(a|x)}{\sqrt{\pi_l(a|x)}}$ and $\vec{n}_a \coloneqq \sqrt{\pi_l(a|x)}$. Because $\pi_l(\cdot|x)$ is a valid probability distribution, we have $||\vec{n}|| = 1$, and we have $||\vec{m}|| \coloneqq C(\pi_l)$. By the Cauchy-Schwarz inequality that $||\vec{m}|| \;||\vec{n}|| \geq |\vec{m}\cdot\vec{m}|=1$, with equality holding iff $\vec{m}_a = c\vec{n}_a \; \forall a \in \mathcal{A}$. Because $\pi_l(\cdot|x)$ and $\pi_t(\cdot|x)$ both comprise valid probability distributions, this equality can only hold when $\pi_l(\cdot|x) = \pi_t(\cdot|x)$.
    \end{proof}

    Now, we examine each case.
    \begin{enumerate}[
      label=\textbf{Case (\arabic*):},
      leftmargin=*,
      align=left
    ]
    \item \textbf{$\pi_l \neq \pi_t$.}
    Denote the set of $\mu$ values constrained to a single value across actions as $\mathcal{M}_C \coloneqq \{\mu(a) = t, \; \forall \; a \in \mathcal{A}: t\in[0,1]\}$. We have that unconstrained supremum will be greater than or equal to the supremum attained on the constrained optimization when all actions have the same value in $\mu$; that is, $\sup_{\mu \in \mathcal{M}}v_{\text{IPW}}(\pi_t, \pi_l | x)\geq \sup_{\mu \in \mathcal{M}_C}v_{\text{IPW}}(\pi_t, \pi_l | x)$.
    
    We now solve this constrained problem. With $\mu \in \mathcal{M}_C$, we have $m(x) = t$ and $v_{IPW}(\pi_t, \pi_l | x) = tC(\pi_l|x)-t^2 $ for some $t \in [0,1]$. By Lemma \ref{lem:cauchy-schwarz}, we have that $C(\pi_l|x)>1$. This leads to $ \max_{t\in [0,1]}tC(\pi_l|x)-t^2>\tfrac{1}{4}$. Therefore, we have that
    \[
    \sup_{\mu \in \mathcal{M}} v_{\text{IPW}}(\pi_t,\pi_l \mid x) > \frac{1}{4}.
    \]

    \item \textbf{$\pi_l = \pi_t$.}
    With $\pi_l = \pi_t$, we have $\sum_{a \in \mathcal{A}}\frac{\pi_t(a|x)^2}{\pi_l(a|x)}=v(x)$. This gives conditional variance of $v_{\text{IPW}} = v(x)-v(x)^2$, which is maximized at $v(x) = \tfrac{1}{2}$. This yields
     \[
    \sup_{\mu \in \mathcal{M}}v_{\text{IPW}}(\pi_t, \pi_l | x)=\frac{1}{4}.
    \]     
  \end{enumerate}
        
With the variance term of $\mathrm{MSE}[\hat{V}_{\text{IPW}}(\pi_t | \pi_l)]$ as a linear combination of $v_{\text{IPW}}(\pi_t, \pi_l | x)$, we have $\pi^*_l = \pi_t$ as the unique minimizer to $\max\limits_{\mu \in \mathcal{M}}\min\limits_{\pi_l \in \Pi} \;\mathrm{MSE}[\hat{V}_{\text{IPW}}(\pi_t | \pi_l)]$.
\end{proof}

\subsection{Proof of Corollary \ref{cor:preservation_of_ordering}}
\label{app:proof_preservation_of_ordering}

\begin{proof}
    We begin by bounding the value of $c$ which will be used to prove that the original inequality holds.
    \begin{lemma}[Constraint lower bound]\label{lem:constraint}
        The normalizing constant meets the constraint $c\geq\mu(a, x)\mu(a', x) \quad \forall \quad a,a' \in \mathcal{A}, x\in \mathcal{X}$.
    \end{lemma}
    \begin{proof}[Proof of Lemma \ref{lem:constraint}]
      Let $a_i$ denote the $i^{th}$ highest average reward for a context, such that $\mu(a_1|x)\geq\mu(a_2|x)\geq...$.  Now, let $c_0\coloneqq \mu(a_1|x)\mu(a_2|x)$. We have that
      \begin{align}
          \frac{ \mu(a_1|x)}{ c_0 + \mu(a_1|x)^2} + \frac{ \mu(a_2|x)}{ c_0 + \mu(a_2|x)^2} &= \frac{ \mu(a_1|x)}{ \mu(a_1|x)\mu(a_2|x) + \mu(a_1|x)^2} + \frac{ \mu(a_2|x)}{ \mu(a_1|x)\mu(a_2|x)  + \mu(a_2|x)^2} \\
          &= \frac{1}{\mu(a_2|x) + \mu(a_1|x)} + \frac{1}{ \mu(a_1|x)  + \mu(a_2|x)}\\
          &= \frac{2}{\mu(a_2|x) + \mu(a_1|x)} \\
          &\geq1
      \end{align}
      as $\mu(a, x)$. Let $g(c)\coloneqq \sum_{a \in \mathcal{A}} \frac{\mu(a, x)}{c + \mu(a, x)^2}$. Thus, we have 
      \begin{align}
        g(c_0) &= \sum_{a \in \mathcal{A}} \frac{\mu(a, x)}{c + \mu(a, x)^2}\\
         &= \frac{2}{\mu(a_2|x) + \mu(a_1|x)}+\sum_{a \in \mathcal{A}\setminus \{a_1,a_2\} } \frac{\mu(a, x)}{c + \mu(a, x)^2}\\
         &\geq 1.
      \end{align}
    Observe that $g$ is decreasing in its argument, so our normalizing constant must meet $c \geq c_0$ to attain $g(c) = 1$. Thus, $c \geq c_0 \geq  \mu(a_1|x)\mu(a_2|x) \geq \mu(a)\mu(a') \quad \forall \quad a,a' \in \mathcal{A}$. 
    \end{proof}
We use this bound on $c$ to prove the monotonicity argument. Let $\mu(a, x) > \mu(a', x)$. Now, we use simple algebraic manipulation to arrive at our claim:
\begin{align}
    c[\mu(a, x)-\mu(a', x)] &> [\mu(a, x)\mu(a', x)][\mu(a, x)-\mu(a', x)]\\
    \iff c[\mu(a, x)-\mu(a', x)] + [\mu(a, x)\mu(a', x)][\mu(a', x)-\mu(a, x)] &> 0\\
    \iff \mu(a, x)[c+\mu(a', x)^2]-\mu(a', x)[c+\mu(a, x)^2]&>0\\
    \iff \frac{\mu(a, x)}{c + \mu(a, x)^2} &> \frac{\mu(a', x)}{c + \mu(a', x)^2}.  
\end{align}
\end{proof}

\subsection{Proof of Proposition \ref{prop:known_mu_known_pit}}
\label{app:proof_known_mu_known_pit}

\begin{proof}
    Consider, again, the single-instance variance $v$. As $\sum_{a \in \mathcal{A}} \pi_t(a)\mu(a)$ is fixed, we have the constrained optimization problem:
    \begin{align*}
        \text{minimize} &\sum_{a \in \mathcal{A}} \frac{\pi_t(a \mid x)^2}{\pi_l(a \mid x)}\mu(a, x)\\
        \text{s.t.} & \sum_{a \in \mathcal{A}}\pi_l(a|x) = 1.
    \end{align*}

First, let $c(a) \coloneqq \pi_t^2(a|x)\mu(a, x)$. Define the Lagrangian

\begin{align}
    \mathcal{L}(\pi_l, \lambda) = \sum_{a \in \mathcal{A}}\frac{c(a)}{\pi_l(a \mid x)} - \lambda(\sum_{a \in \mathcal{A}}\pi_l(a|x) -1).
\end{align}
Now, we set the KKT conditions. For every action \(a\),
\begin{equation}
\frac{\partial\mathcal L}{\partial \pi_l(a|x)}=0
\;\Longrightarrow\;
-\frac{c_a}{\pi^*_l(a|x)^2}+\lambda=0
\;\Longrightarrow\;
\pi^*_l(a|x)=\sqrt{\frac{c_a}{\lambda}}.
\end{equation}
We now impose $\sum_{a \in \mathcal{A}}\pi^*_l(a|x)=1$ to identify the Lagrange multiplier:
\begin{equation}
1=\sum_{a\in\mathcal{A}}\pi^*_l(a|x) 
  \;=\;\sum_{a \in \mathcal{A}}\sqrt{\frac{c_a}{\lambda}}
  \;\Longrightarrow\;
  \sqrt{\lambda}\;=\;\sum_{a'\in\mathcal A_i}\sqrt{c_{a'}}.
\end{equation}
Substituting $\lambda$ back, we have
\begin{equation}
\pi_l^*(a|x)
  \;=\;
  \frac{\sqrt{c_a}}{\sum_{a'\in\mathcal{A}_x}\sqrt{c_{a'}}}
  \;=\;
  \frac{\pi_t(a|x)\sqrt{\mu(a, x)}}{\sum_{a'\in\mathcal A}\pi_t(a'|x)\sqrt{\mu(a', x)}}.
\end{equation}
Given that this solution minimizes conditional variance, it minimizes the weighted sum of these conditional variance terms. Therefore, $\pi_l^*(a|x) =
  \frac{\pi_t(a|x)\,\sqrt{\mu(a, x)}}{\sum_{a'\in\mathcal A}\pi_t(a'|x)\,\sqrt{\mu(a', x)}}$ minimizes $\mathrm{MSE}[\hat{V}_{\text{IPW}}(\pi_t | \pi_l)]$.
\end{proof}

\subsection{Proof of Corollary \ref{cor:deterministc_pit_pil} }
\label{app:proof_deterministc_pit_pil}

\begin{proof}\label{prf:deterministc_pit_pil}
    This result follows from Proposition \ref{prop:known_mu_known_pit}. With deterministic $\pi_t$, we have that 
    \begin{equation*}
        \sum_{a'\in\mathcal A}\pi_t(a'|x)\,\sqrt{\mu(a', x)} = \sqrt{\mu(a^*|x)} 
    \end{equation*}
    with $a^*$ being the action deterministically selected by $\pi_t$. Then, 
    \begin{equation*}
        \pi_l(a|x) = \frac{\pi_t(a|x) \sqrt{\mu(a, x)}}{ \sqrt{\mu(a, x)}}. 
     \end{equation*}
     So, $\pi_l(a^*|x) = 1$ and $\pi_l(a'|x) = 0 \quad \forall a' \neq a$. 
\end{proof}

\subsection{Proof of Corollary \ref{cor:on_policy_off_policy} }
\label{app:proof_on_policy_off_policy}

\begin{proof}
    Observe that $\pi_l^*$ induces an unbiased estimate $\hat{V}_{\text{IPW}}(\pi_t|\pi_l^*)$. We have conditional MSE captured entirely by the conditional variance $v_{\text{IPW}}(\pi_t,\pi_l^*|x)$, defined earlier. Let $Z \coloneqq \sum_{a \in \mathcal{A}_x}\pi_t(a|x)\sqrt{\mu(a, x)}$  and $\bar{\mu} \coloneqq \sum_{a \in \mathcal{A}_x} \pi_t(a \mid x)\mu(a, x)$. Then
\begin{align}
    v_{\text{IPW}}(\pi_t,\pi_l^* | x) &=\left(\sum_{a \in \mathcal{A}_x} \frac{\pi_t(a \mid x)^2}{\pi_l^*(a \mid x)} \mu(a, x)\right)-\left(\sum_{a \in \mathcal{A}_x} \pi_t(a \mid x)\mu(a, x)\right)^2\\
&=\left(\sum_{a \in \mathcal{A}_x} \frac{\pi_t(a \mid x)^2\sum_{a' \in \mathcal{A}_x}\pi_t(a'|x)\sqrt{\mu(a', x)}}{\pi_t(a \mid x)\sqrt{\mu(a, x)}} \mu(a, x)\right)-\bar{\mu}^2\\
&=\left(Z\sum_{a \in \mathcal{A}_x} \pi_t(a \mid x)\sqrt{\mu(a, x)}\right)-\bar{\mu}^2\\
&= Z^2 -\bar{\mu}^2
    \end{align}
Letting $u \coloneqq \left(\sqrt{\pi_t(a|x)}\right)_{a\in\mathcal{A}_x}$ and $v  =\left(\sqrt{\pi_t(a|x)}\sqrt{\mu(a, x)}\right)_{a\in\mathcal{A}_x}$, we have by Cauchy-Schwartz
\begin{align*}
    |\langle u,v\rangle|^2 &\leq \langle u,u \rangle \cdot \langle v,v \rangle\\
    Z^2&\leq \sum_{a \in \mathcal{A}_x}\pi_t(a|x)\sum_{a \in \mathcal{A}_x}\pi_t(a|x)\mu(a, x).\\
    \end{align*}
    Recall, that $\pi_t(\cdot|x)$ comprises a probability distribution, so this simplifies to
    \begin{align}
    Z^2&\leq 1\cdot\sum_{a \in \mathcal{A}_x}\pi_t(a|x)\mu(a, x)\\
      Z^2&\leq \bar{\mu}.
    \end{align}

Now, let $v_{\text{emp}}(x)$ be the conditional variance in reward given $X=x$. This is defined as
\begin{align*}
v_{\text{emp}}(x)&\coloneqq \text{Var}[R\mid X=x] \\
    &= \mathbb{E}\left[\text{Var}(R\mid X=x)\right] 
       + \text{Var}\left(\mathbb{E}[R\mid X=x]\right) \\
    &= \sum_{a \in \mathcal{A}_x}\pi_t(a|x)\mu(a, x)\left(1-\mu(a, x)\right) 
       + \sum_{a \in \mathcal{A}_x}\pi_t(a|x)\mu(a, x)^2 
       - \left(\sum_{a \in \mathcal{A}_x}\pi_t(a|x)\,\mu(a, x)\right)^2 \\
    &= \sum_{a \in \mathcal{A}_x}\pi_t(a|x)\mu(a, x) 
       - \left(\sum_{a}\pi_t(a|x)\,\mu(a, x)\right)^2 \\
    &= \bar{\mu} - \bar{\mu}^{\,2}.
\end{align*}

Using these on (9), we have 
\begin{align}
    v_{\text{IPW}}(\pi_t,\pi_l^* | x)&= Z^2 -\bar{\mu}^2\\
    &\leq \bar{\mu} - \bar{\mu}^2 =  v_{\text{emp}}(x).
\end{align}
    This inequality extends to the aggregate level as the inequality holds across exogenous context arrivals, which linearly combines $v(x)$, and sample sizes, divides both by $n$. Therefore, $\mathrm{MSE}[\hat{V}_{\text{IPW}}(\pi_t|\pi_l^*)] \leq \mathrm{MSE}[\hat{V}_{\text{emp}}(\pi_t)]$. 
\end{proof}

\subsection{Proof of Corollary \ref{cor:safety_off_policy}}
\begin{proof}
    We have the target and logging policy values
    \begin{align*}
        V(\pi_t) &= \sum_{x \in \mathcal{X}} \Pr(x) \left[\sum_{a \in \mathcal{A}}\pi_t(a|x)\mu(a, x)\right]\\
        V(\pi_l^*) &= \sum_{x \in \mathcal{X}} \Pr(x) \left[\sum_{a \in \mathcal{A}} \frac{\pi_t(a|x)\sqrt{\mu(a, x)}}{\sum_{a'\in\mathcal A}\pi_t(a'|x)\sqrt{\mu(a', x)}}\mu(a, x)\right]\\
        &=  \sum_{x \in \mathcal{X}} \Pr(x) \left[\frac{\sum_{a \in \mathcal{A}}\pi_t(a|x)\mu(a, x)^{3/2}}{\sum_{a\in\mathcal A}\pi_t(a|x)\mu(a, x)^{1/2}}\right]
    \end{align*}

    To simplify notation, define $Y(x)$ assume the value $\sqrt{\mu(a, x)}$ w.p. $\pi_t(a|x)$. Then we have
    \begin{align*}
        V(\pi_t) &= \sum_{x \in \mathcal{X}} \Pr(x)\mathbb{E}[Y(x)^2]\\
        V(\pi_l^*) &=  \sum_{x \in \mathcal{X}} \Pr(x) \frac{\mathbb{E}[Y(x)^3]}{\mathbb{E}[Y(x)]}.
    \end{align*}

    Now, let's examine the conditional value. We will show this conditional value is at least as great for a given context $x$. Comparing the differences in these conditional values, we have 
    \begin{align}
        \frac{\mathbb{E}[Y(x)^3]}{\mathbb{E}[Y(x)] }-\mathbb{E}[Y(x)^2] &= \frac{\mathbb{E}[Y(x)^3]-\mathbb{E}[Y(x)^2]\mathbb{E}[Y(x)]}{\mathbb{E}[Y(x)]}\\
        &= \frac{Cov(Y(x),Y(x)^2)}{\mathbb{E}[Y(x)]}.
    \end{align}
Because $\mu$ assumes only non-negative values, $\text{Cov}(Y(x),Y(x)^2) \geq 0$, and $\mathbb{E}[Y(x)]\geq0$ .  Therefore, we have 
\begin{align*}
        V(\pi_l^*) -V(\pi_t) &= \sum_{x \in \mathcal{X}} \Pr(x)\frac{\text{Cov}(Y(x),Y(x)^2)}{\mathbb{E}[Y(x)]}\\
         &\geq 0.
    \end{align*}
    If we make the assumption that the support of $Y(x)$ is non-zero for at least one $x$, this inequality becomes strict.
    
\end{proof}

\subsection{Proof of Proposition \ref{prop:expected_pit}}
\begin{proof}
    \begin{align}
        \text{argmin}_{\pi_l}\mathbb{E}_{\pi_t \sim \Pi_t} \ \mathrm{MSE}[\hat V_{\text{IPW}}(\pi_t|\pi_l)] &= \text{argmin}_{\pi_l}\text{Variance}(
        \pi_t|\pi_l)\\
        &=\text{argmin}_{\pi_l} \sum_{\pi_t \in \Pi_t}\Pr(\pi_t)\sum_{x \in \mathcal{X}}\Pr(x)\sum_{a \in \mathcal{A}}\frac{\pi_t(a|x)^2}{\pi_l(a|x)}\mu(a, x)\\
        &=\text{argmin}_{\pi_l}\sum_{x \in \mathcal{X}}\Pr(x)\sum_{a \in \mathcal{A}}\frac{\sum_{\pi_t \in \Pi_t}\Pr(\pi_t)\pi_t(a|x)^2}{\pi_l(a|x)}\mu(a, x)
    \end{align}
    Following similar steps to the proof of the optimal $\pi_l$ under known $\pi_t$, let $c(a) \coloneqq \sum_{\pi_t \in \Pi_t}\Pr(\pi_t)\pi_t(a|x)^2\mu(a, x)$. 
    
    Define the Lagrangian
\begin{align}
    \mathcal{L}(\pi_l, \lambda) = \sum_{a \in \mathcal{A}}\frac{c(a)}{\pi_l(a \mid x)} - \lambda(\sum_{a \in \mathcal{A}}\pi_l(a|x) -1).
\end{align}
Set the KKT conditions. For every action \(a\),
\begin{equation}
\frac{\partial\mathcal L}{\partial \pi_l(a|x)}=0
\;\Longrightarrow\;
-\frac{c(a)}{\pi^*_l(a|x)^2}+\lambda=0
\;\Longrightarrow\;
\pi^*_l(a|x)=\sqrt{\frac{c(a)}{\lambda}}.
\end{equation}
We now impose $\sum_{a \in \mathcal{A}}\pi^*_l(a|x)=1$ to identify the Lagrange multiplier:
\begin{equation}
1=\sum_{a\in\mathcal{A}}\pi^*_l(a|x) 
  \;=\;\sum_{a \in \mathcal{A}}\sqrt{\frac{c(a)}{\lambda}}
  \;\Longrightarrow\;
  \sqrt{\lambda}\;=\;\sum_{a'\in\mathcal A_i}\sqrt{c(a')}.
\end{equation}
Substituting $\lambda$ back, we have
\begin{equation}
\pi_l^*(a|x)
  \;=\;
  \frac{\sqrt{c(a)}}{\sum_{a'\in\mathcal{A}_i}\sqrt{c(a)}}
  \;=\;
  \frac{\sqrt{\sum_{\pi_t \in \Pi_t}\Pr(\pi_t)\pi_t(a|x)^2\mu(a, x)}}{\sum_{a'\in\mathcal A}\sqrt{\sum_{\pi_t \in \Pi_t}\Pr(\pi_t)\pi_t(a'|x)^2\mu(a', x)}}.
\end{equation}
Which can be alternatively rewritten as 
\begin{equation}
\pi_l^*(a|x)
  \;=\;
  \frac{\sqrt{\mathbb{E}[\pi_t(a|x)^2]\mu(a, x)}}{\sum_{a'\in\mathcal A}\sqrt{\mathbb{E}[\pi_t(a'|x)^2]\mu(a', x)}}.
\end{equation}
\end{proof}

\clearpage
\newpage
\section{Derivation of MSE($\hat{V}_{IPW}$)}\label{app:b_v_decomposition}
We have that 
\begin{align*}
    \mathrm{MSE}[\hat{V}(\pi_t|\pi_l)]&=(V(\pi_t)-\hat{V}(\pi_t|\pi_l))^2.
\end{align*}
We can decompose this in accordance to the bias-variance decomposition to arrive at the MSE. 

Define $\mathcal{A}_x$ to be the set of actions for $x$ where $\forall a \in \mathcal{A}_{x}, \;\pi_l(a,x)>0$, then we have:
\begin{align*}
\text{Bias}^2 &= \left(V(\pi_t)-\mathbb{E}[\hat V_{\text{IPW}}(\pi_t|\pi_l)] \right)^2\\
&= \left(\sum_{x \in \mathcal{X}} \Pr(x) \left(\sum_{a \in \mathcal{A}} \pi_t(a \mid x) \mu(a, x) - \sum_{a \in \mathcal{A}_x} \pi_l(a|x)\frac{\pi_t(a \mid x)}{\pi_l(a|x)} \mu(a, x) \right)\right)^2\\
&= \left(\sum_{x \in \mathcal{X}} \Pr(x)\sum_{a \notin \mathcal{A}_x} \pi_t(a \mid x_i) \mu(a, x)\right)^2\\
&= \left( \sum_{x \in \mathcal{X}} \Pr(x)\sum_{a \notin \mathcal{A}_x} \pi_t(a \mid x) \mu(a, x)\right)^2
\end{align*}
For notational compactness, let $Z_x \coloneqq \frac{\pi_t(A|x)}{\pi_l(A|x)}R$ with randomness due to action selection and Bernoulli draws. With the independence of $Z$ terms in variance, we can decompose this term in the following way:
\begin{align*}
    \text{Variance} &= \mathbb{E}\left[\left(\mathbb{E}[\hat{V}(\pi_t|\pi_l)]-\hat{V}(\pi_t|\pi_l)\right)^2\right]\\
    &= \frac{1}{N}\sum_{x \in \mathcal{X}} \Pr(x)\text{Var}(Z_x)\\
    &= \frac{1}{N}\sum_{x \in \mathcal{X}} \Pr(x)(\mathbb{E}[Z_x^2]-\mathbb{E}[Z_x]^2)\\
\end{align*}
Because we have $\mathbb{E}[Z_x]
=\sum_{a\in \mathcal{A}_x} \pi_\ell(a\mid x)\,\frac{\pi_t(a\mid x)}{\pi_\ell(a\mid x)}\,\mu(a,x)
= \sum_{a\in \mathcal{A}_x} \pi_t(a\mid x)\,\mu(a,x)$. 

Now,
\begin{align*}
       \text{Variance} &= \frac{1}{N}\sum_{x \in \mathcal{X}} \Pr(x)\left(\mathbb{E}[Z_x^2]-\sum_a \pi_t(a\mid x)\mu(a,x)\right)\\
       &= \frac{1}{N}\sum_{x \in \mathcal{X}} \Pr(x)\left(\left(\sum_{a \in \mathcal{A}_i}\frac{\pi_t(a| x)^2}{\pi_l(a| x)}\mu(a,x)\right) -\left(\sum_{a \in \mathcal{A}_i} \pi_t(a| x_i)\mu(a,x) \right)^2 \right)
\end{align*}
This leads to a final, closed-form expression of MSE.

\begin{align*}
    \mathrm{MSE}[\hat{V}(\pi_t|\pi_l)]&= \left( \sum_{x \in \mathcal{X}}  \Pr(x) \sum_{a \notin \mathcal{A}_i}\pi_t(a \mid x) \mu(a, x)\right)^2 \\
    &+\frac{1}{N}\left[
    \sum_{x \in \mathcal{X}} \Pr(x) \left(\left(\sum_{a \in \mathcal{A}_x} \frac{\pi_t(a \mid x)^2}{\pi_l(a \mid x)} \mu(a, x)\right)-\left(\sum_{a \in \mathcal{A}_x} \pi_t(a \mid x)\mu(a, x)\right)^2\right)\right].
\end{align*}

\clearpage
\newpage
\section{Example of Bias Improving Estimate}\label{app:examples}
There are cases, however, where worst-case MSE can be minimized by not sampling actions that might be sampled by the target policy, inducing a biased estimate.

\begin{proposition}\label{prop:bias_error_reduction}
    For an unconstrained target policy class $\Pi$, a logging policy with full support does not always minimize $\max\limits_{\mu \in \mathcal{M}, \pi_t \in \Pi} \;\mathrm{MSE}[\hat{V}_{\text{IPW}}(\pi_t | \pi_l)]$.
\end{proposition}
\begin{proof}
    Let $|\mathcal{A}| = 10$ and $|\mathcal{X}| = 1$ and $n = 1$. Consider a deviation from $\pi_U$ to the degenerate policy $\pi'(a_1|x_1) = 1$.

   For $\pi_U$ we have:
   \begin{align*}
       \max\limits_{\mu \in \mathcal{M}, \pi_t \in \Pi} \;\mathrm{MSE}[\hat{V}_{\text{IPW}}(\pi_t | \pi_U)] &=\max\limits_{\mu \in \mathcal{M}, \pi_t \in \Pi}\left[\sum_{a \in \mathcal{A}}|\mathcal{A}|\pi^2_t(a|x_1)\mu(a, x_1)\right. \\
       &-\left(\left.\sum_{a \in \mathcal{A}}\pi_t(a|x_1)\mu(a, x_1)\right)^2\right] \\
       &= 9.
   \end{align*}

   However, for $\pi_l'$, we have:
     \begin{align*}
       \max\limits_{\mu \in \mathcal{M}, \pi_t \in \Pi} \;\mathrm{MSE}[\hat{V}_{\text{IPW}}(\pi_t | \pi_l')]&=\max\limits_{\mu \in \mathcal{M}, \pi_t \in \Pi}\bigg[\sum_{a \neq a_1 }\pi_t(a|x_1)\mu(a, x_1) \\
       &+\pi^2_t(a_1|x_1)\mu(a_1|x_1) -(\pi_t(a_1|x_1)\mu(a_1|x_1))^2\bigg] \\
       &= 1.
   \end{align*}
So, there exists conditions under which the minimax MSE is minimized via a biased policy.
   
\end{proof}

\clearpage
\newpage
\section{Simulations}\label{app:simulations}
\subsection{Additional Simulations}
\subsubsection{Other Policy Mappings}
We briefly extend the results of Figure \ref{fig:softgreedy_errors} to other policy classes, motivated by the performance of top-$k$ policies with a well chosen $k$. Specifically, we use top-$k$ to truncate logging propensities. Then, rather than sample uniformly across these actions, we sample in accordance to the estimated reward probability. Thus, we end up with two parameters of greediness for the following classes:
\paragraph{Power-normalized across top-k (Top-k PN).} We defined this class to be $\Pi_{TKPN}(\hat{\mu}) \coloneqq \Big\{\pi_{k,d}; \forall d \in [0,\infty), k \in 1...|\mathcal{A}| ]\Big\}$,  where $\pi_{k,d}$ is the power-normalization function using degree $d$ described in Section \ref{sec:Discussion}, however where normalization only happens over the top-$k$ items.
\paragraph{Softmax across top-k (Top-k SM).} We defined this class to be $\Pi_{TKSM}(\hat{\mu}) \coloneqq \Big\{\pi_{k,\alpha}; \forall \alpha \in [0,\infty), k \in 1...|\mathcal{A}| ]\Big\}$,  where $\pi_{k,\alpha}$ is the softmax function using parameter $\alpha$ described in Section \ref{sec:Discussion},  where sampling only happens over the top-$k$ items.

In Figure \ref{fig:softgreedy_errors_more_policies}, we compare these policies to those in Figure \ref{fig:softgreedy_errors}. We plot these additional classes at several different $d$ or $\alpha$ parameters respectively across all $k$ values. In general, these classes seem to improve performance as $k$ grows large, suggesting these can limit the downside of selecting a support which is too large. Specifically, we note that for the given simulation parameters, sampling across the top-$k$ proportional to raw reward estimates ($d=1$) can yield comparable upside MSE when $k$ is tuned well but limit MSE loss as $k$ grows too large. 
\begin{figure}[h!]
\FIGURE
    {
    \centering
    \begin{subfigure}[b]{0.475\textwidth}
        \includegraphics[width=\linewidth]{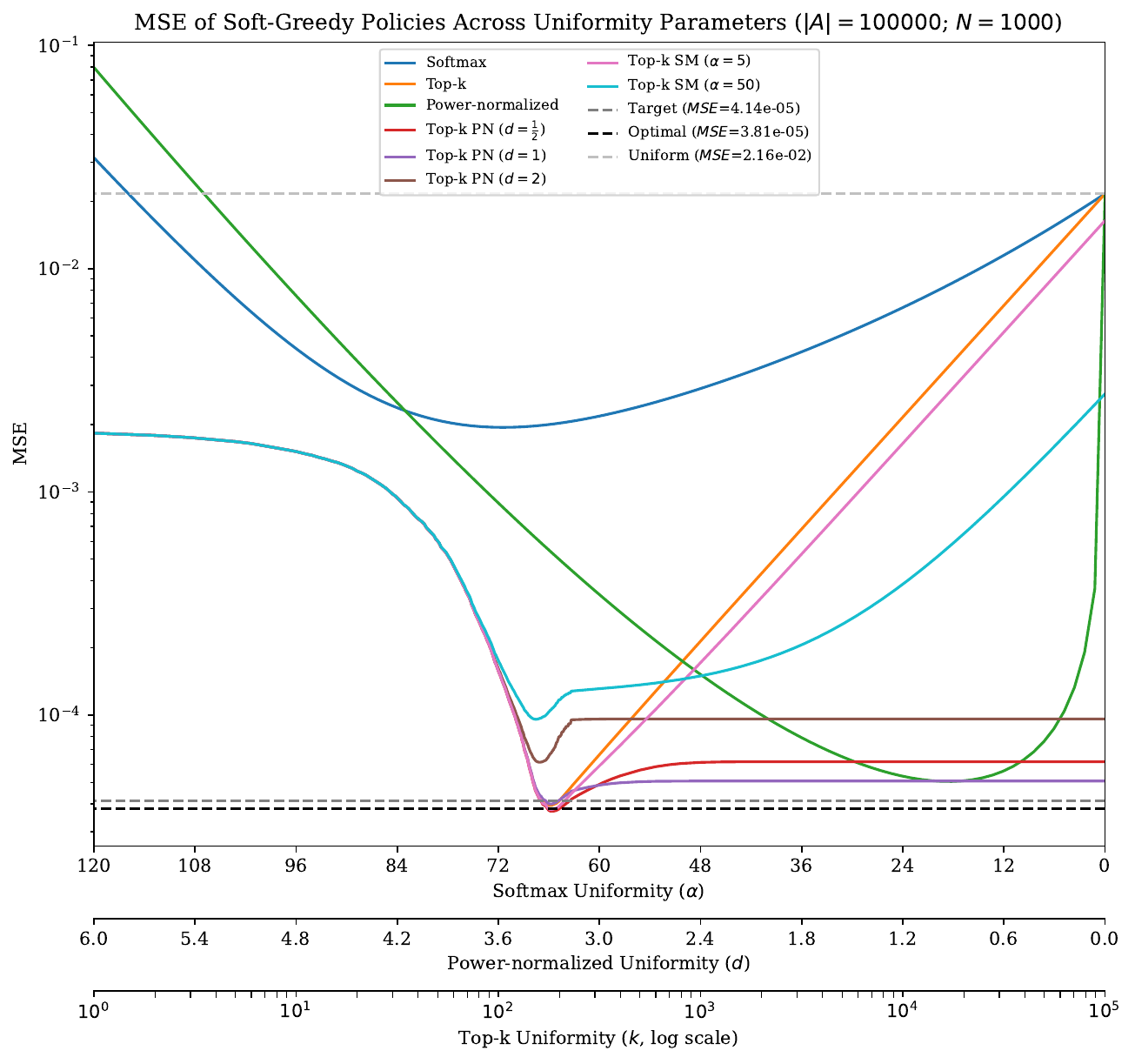}
        \caption{\textit{Small action space} $|\mathcal{A}|=1,000$.} 
    \end{subfigure}
    \hfill
    \begin{subfigure}[b]{0.475\textwidth}
        \includegraphics[width=\linewidth]{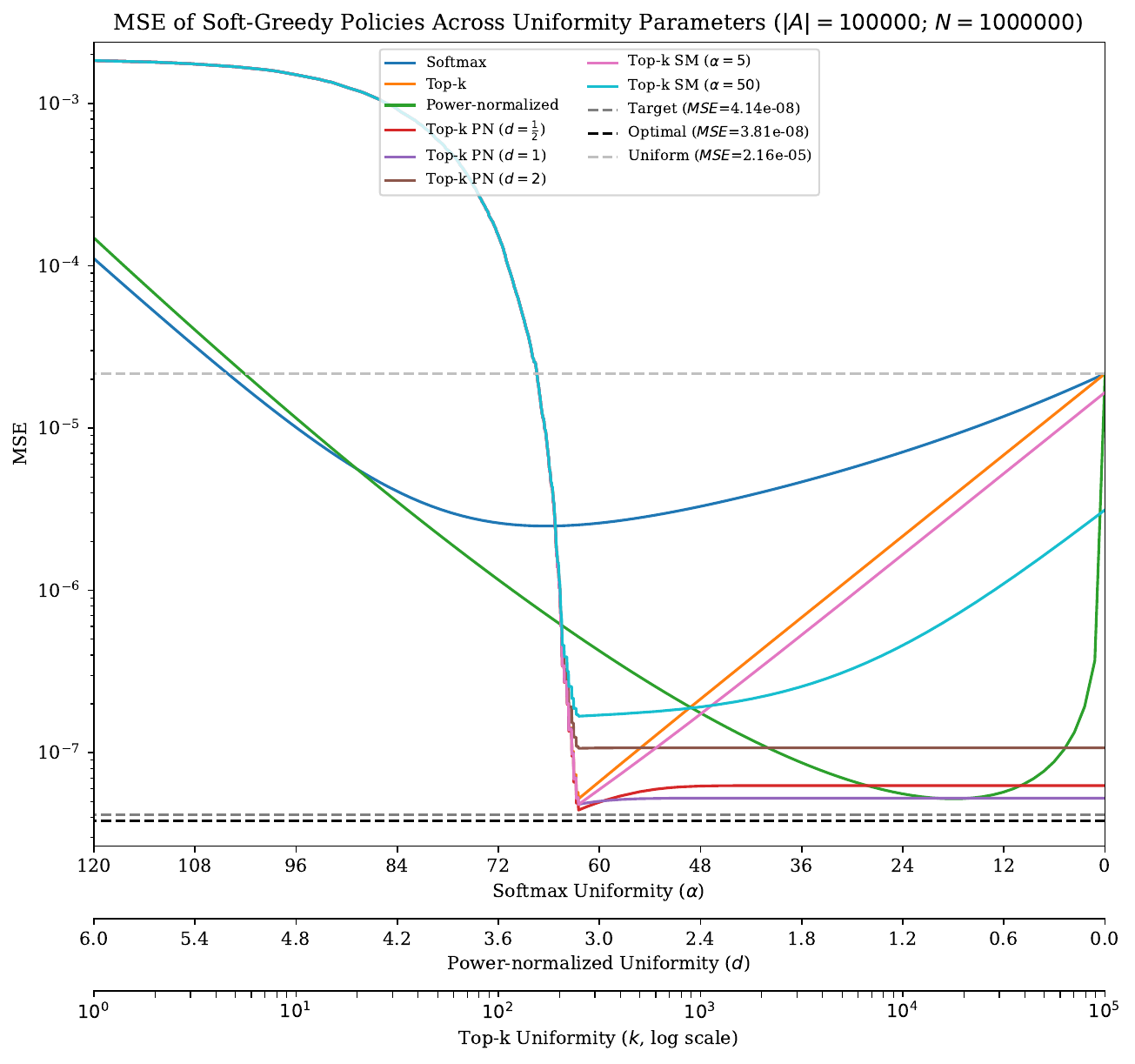}
        \caption{\textit{Large action space} $|\mathcal{A}|=100,000$.}
    \end{subfigure}
    }
    {
    MSE of IPW estimator for soft-greedy logging policy classes \textit{top-$k$} and \textit{softmax}. 
    \label{fig:softgreedy_errors_more_policies}
    }
    {
     Simulation evidence of the MSE of the IPW estimator as a function of the greediness parameter for the \textit{top-$k$}, \textit{softmax}, and \textit{power-normalized} logging policy classes, in addition to \textit{top-k PN} and \textit{top-k SM}. 
     Full simulation details are in Appendix \ref{app:simulations}, under the same parameters as Figure \ref{fig:softgreedy_errors}.
    }
\end{figure}

\subsection{Environment Description}
In constructing the simulation environment, we set the number of users $|\mathcal{X}|$ and the number of actions $|\mathcal{A}|$. Then, for each $x \in \mathcal{X}$, we assign some arrival probability $\Pr(x)$. We then generate reward probabilities $\mu(a,x)$ as follows. For each user context $x$, randomly permute $\mathcal{A}$ to generate $\mathcal{A}'$. Then, for $i \in 1...|\mathcal{A}|$, assign $\mu( \mathcal{A}'_i,x)=\alpha\gamma^i$, where $\alpha$ is a scale parameter and $\gamma \in [0,1]$ is the decay rate. Model estimates $\hat{\mu}$ are generated from $\mu$  by multiplying each underlying probability by noise $\mathcal{N}(1,\sigma^2)$. At simulation time, the dataset is constructed as follows. For each instance, a user is drawn from $x \sim \Pr(x)$ with replacement, an action is taken according to the logging policy $a \sim \pi_l(\cdot|x)$, and a reward is drawn from the underlying distribution $r \sim \text{Bernoulli}(\mu(a,x))$. This is then repeated $n$ times to form a dataset.

\subsubsection*{Simulation generating Figure \ref{fig:estimate_histogram} } Simulation parameters are: $|\mathcal{A}|=10,000$  for a single context, with scaled geometric reward preferences ($\gamma = .99$, $\alpha=\frac{1}{10}$). The target policy is a \textit{top-$k$} policy ($k=10$) on $0$-noise estimates of $\mu$. The personalized logging policy is a  \textit{top-$k$} policy ($k=10$) on $\hat{\mu}$ with multiplicative Gaussian noise of $\sigma^2 = .25$.

\subsubsection*{Simulation generating Figure \ref{fig:mse_in_noise}}Simulation parameters are: $|\mathcal{A}|=1,000$ a single context ($|\mathcal{X}| = 1$), with scaled geometric reward preferences ($\gamma = .99$, $\alpha=\tfrac{1}{10}$). There are $100$ samples drawn for all logging policies to generate the MSE. Noises in $\hat{\mu}_l$ are placed equally spaced on $[0,.25]$, with $1,000$ trials at each noise. The target policy is generated for each trial based on a \textit{top-$k$} policy ($k=30$) on noisy estimates of $\mu$ with ($\epsilon = .25$). The optimal logging policy is derived for this target policy for each trial.

\subsubsection*{Simulation generating Figure \ref{fig:shrinkage}} Simulation parameters are: $|\mathcal{A}| = 1,000$, $|\mathcal{X}| = 1,000$ with $\mu(a,x)$ distributed by randomly permuting $\mathcal{A}$ then assigning a linearly decreasing $\mu(a,x)$ on the equally spaced lattice from $.4$ to $0$. Here, $50,000$ samples are drawn for each trial. The target policy is a \textit{top-$k$} policy ($k=100$) based on estimates generated with standard multiplicative noise with $\sigma^2 = .05$. Both MSE and observed outcomes used to derive $w^*$ are based on a sample size of $50,000$. 

\subsubsection*{Simulation generating Figure \ref{fig:softgreedy_errors}} Simulation parameters are: $|\mathcal{A}|=1,000$  for the small action space setting and $|\mathcal{A}|=100,000$ for the large action space setting for a single context, with scaled geometric reward preferences ($\gamma = .99$, $\alpha=\tfrac{1}{10}$), with. $1,000$ samples are drawn for all logging policies. The target policy is a \textit{top-$k$} policy ($k=200$) on $0$-noise estimates of $\mu$.

\end{APPENDICES}

\end{document}